%% file: main.tex
\documentclass{article}

\newif\ifnoappendix
\noappendixfalse

\newif\ifpreprint
\preprinttrue

\ifpreprint
\usepackage[nonatbib,preprint]{neurips}
\else
\usepackage[nonatbib]{neurips}
\fi

\usepackage[utf8]{inputenc} 
\usepackage[T1]{fontenc}    
\usepackage{hyperref}       
\usepackage{url}            
\usepackage{booktabs}       
\usepackage{amsfonts}       
\usepackage{microtype}      
\usepackage{xcolor}         
\usepackage{enumitem}       
\usepackage{dsfont}			

\usepackage{numprint}       

\usepackage{algorithm}
\usepackage{cite}
\usepackage{amsthm,amsmath,amssymb}
\usepackage{mathtools}

\mathtoolsset{showonlyrefs}

\usepackage{graphicx}
\usepackage{textcomp}
\usepackage{multirow}
\usepackage{caption}
\usepackage{subcaption}
\usepackage{physics}

\usepackage{thmtools}
\usepackage{thm-restate}

\def\BibTeX{{\rm B\kern-.05em{\sc i\kern-.025em b}\kern-.08em
    T\kern-.1667em\lower.7ex\hbox{E}\kern-.125emX}}

\DeclareMathOperator{\sumtau}{\sum_{t=K+1}^{\tau}}

\DeclareMathOperator*{\argmax}{arg\,max}

\newcommand{\avgreward}{\ensuremath{\bar \mu^r_k(t)}}
\newcommand{\avgcost}{\ensuremath{\bar \mu^c_k(t)}}

\newcommand{\expectrew}{\ensuremath{\mu_k^r}}
\newcommand{\expectcost}{\ensuremath{\mu_k^c}}
\newcommand{\rcratio}{\ensuremath{\frac{\omega_{k+}^r(t)}{\omega_{k-}^c(t)}}}

\newcommand{\omegaplus}{\ensuremath{\omega_{k+}^r(t)}}
\newcommand{\omegaminus}{\ensuremath{\omega_{k-}^c(t)}}

\newcommand{\eq}[1]{Eq.~\eqref{#1}}

\theoremstyle{definition}

\newtheorem{example}{Example}

\title{Budgeted Multi-Armed Bandits with\\ Asymmetric Confidence Intervals}

\author{%
  Marco Heyden, Vadim Arzamasov, Edouard Fouch\'{e}, Klemens B\"{o}hm
  \\
  Karlsruhe Institute of Technology \\
  Karlsruhe, Germany \\
  \texttt{\{marco.heyden, vadim.arzamasov, edouard.fouche, klemens.boehm\}@kit.edu } \\
}

\ifnoappendix
\fi

\begin{document}

\include{body}
\include{appendix}

\end{document}

%% file: body.tex
\maketitle

\begin{abstract}

We study the stochastic Budgeted Multi-Armed Bandit (MAB) problem, where a player chooses from $K$ arms with unknown expected rewards and costs. The goal is to maximize the total reward under a budget constraint. A player thus seeks to choose the arm with the highest reward-cost ratio as often as possible. Current state-of-the-art policies for this problem have several issues, which we illustrate. To overcome them, we propose a new upper confidence bound (UCB) sampling policy, $\omega$-UCB, that uses asymmetric confidence intervals. These intervals scale with the distance between the sample mean and the bounds of a random variable, yielding a more accurate and tight estimation of the reward-cost ratio compared to our competitors. We show that our approach has logarithmic regret and consistently outperforms existing policies in synthetic and real settings. 
\end{abstract}

\input{introduction}
\input{problem_definition}
\input{related_work}

\input{our_approach}
\input{theoretical_analysis}
\input{experiments}
\input{conclusion}
\clearpage
\begin{ack}
This work was supported by the DFG Research Training Group 2153: ``Energy Status Data --- Informatics Methods for its Collection, Analysis and Exploitation''.
\end{ack}
\small
\bibliographystyle{plain}
\bibliography{references.bib}
\normalsize

%% file: introduction.tex
\section{Introduction}\label{sec:intro}

In the stochastic Multi-Armed Bandit (MAB) problem, a player repeatedly plays one of $K$ arms and receives a corresponding random reward. The goal is to maximize the cumulative reward by playing the arm with the highest expected reward as often as possible. The expected rewards are initially unknown, so the player must balance trying arms to learn their expected rewards (exploration) versus using the current information to play arms with known high expected rewards (exploitation). 

In the stochastic Budgeted MAB problem~\cite{tran-thanh_epsilon-first_2010}, a player must consider not only the potential rewards but also the associated random costs for each arm. The player chooses arms until the available budget is exhausted. Budgeted MABs model real-world situations such as the selection of a cloud service provider~\cite{ardagna_game_2011}, energy-efficient task selection for battery-powered embedded devices~\cite{tran-thanh_knapsack_2012}, bid optimization~\cite{borgs_dynamics_2007,chakraborty_selective_2010,ben-yehuda_deconstructing_2013}, or optimizing advertising on social media. 

\begin{example}[Social media advertising] \label{ex:adverstising}
Consider a retail company that wants to advertise products on a social network platform. 
The retail company provides to the platform an advertisement campaign consisting of multiple ads, as well as an advertisement budget. 
Each time a user clicks on an ad (an arm), the platform charges the retail company (a cost). Within the given budget, the retailer wants to find the ads which maximize the likelihood of a subsequent purchase (a reward). 
Both the reward and the cost are random variables since they depend on the actions of users and the competition from other advertisers. A Budgeted MAB algorithm can help to find the most promising ads in real time.  
\end{example} 

A variety of policies has been proposed to address the Budgeted MAB problem. Section~\ref{sec:related_work} provides a summary. Many policies extend ideas from traditional multi-armed bandit algorithms, in which the costs of arms are assumed to be constant, including Thompson Sampling~\cite{xia_thompson_2015} and Upper Confidence Bound (UCB) sampling~\cite{xia_finite_2017}. Several studies~\cite{xia_budgeted_2015,xia_budgeted_2016,xia_finite_2017,watanabe_kl-ucb-based_2017,watanabe_ucb-sc_2018} indicate that UCB-sampling policies perform well in practice, in particular when the budget is small. 

UCB-sampling policies continuously update an upper bound of the ratio of the expected rewards and costs of each arm, and play the arm with the highest upper bound. We distinguish between three types: Some policies~\cite{watanabe_kl-ucb-based_2017, watanabe_ucb-sc_2018, xia_finite_2017} compute the bound from the ratio of the sample average reward and average cost plus some uncertainty-related term (cf.~\eq{eq:high-level-ucb}, left). We call this type ``united'' (u). Other policies~\cite{xia_finite_2017, badanidiyuru_bandits_2013, xia_budgeted_2016} divide the reward's upper confidence bound by the cost's lower confidence bound (LCB) (cf.~\eq{eq:high-level-ucb}, right). We refer to this type as ``composite'' (c). There also are ``hybrid'' (h) policies~\cite{xia_finite_2017, xia_budgeted_2015} that combine the united and composite types. 

\begin{equation}\label{eq:high-level-ucb}
    \textit{UCB}_u = \frac{\textit{average reward}}{\textit{average cost}} + \textit{uncertainty} \quad\quad \textit{UCB}_c = \frac{\textit{average reward} + \textit{uncertainty}}{\textit{average cost} - \textit{uncertainty}}
\end{equation}

However, all of the current policies have at least one of the following issues: 
\begin{itemize}[noitemsep]
	\item[(i1)] \textbf{Over-optimism:} The policy often computes UCB that are too tight.
	\item[(i2)] \textbf{Over-pessimism:} The policy often computes UCB that are too loose.
	\item[(i3)] \textbf{Invalid values:} Negative or undefined UCB occur if the cost's lower confidence bound in \eq{eq:high-level-ucb} becomes negative or zero. This can happen, for instance, when computing the lower confidence bound of an arm's expected cost with Hoeffding's inequality~\cite{xia_finite_2017,xia_budgeted_2015}. 
\end{itemize}

To illustrate (i1) and (i2), we randomly parameterized \numprint{10000} Bernoulli reward and cost distributions and sampled from them. We used these samples to compute 99\% confidence intervals of the reward-cost ratio using several state-of-the-art UCB-sampling policies. The left plot of Figure~\ref{fig:issues-illustration} shows the share of cases when the expected reward-cost ratio exceeds (i.e., violates) its UCB. Values above 1\% indicate overly tight bounds. The right plot shows the UCB of the reward-cost ratio divided by its expectation, with higher values indicating looser bounds.
Existing united policies (u) tend to suffer from issue (i1), hybrid approaches suffer from either of both issues, and issue (i2) mainly affects the composite approach (c). 

To address~(i1) and~(i2), some approaches provide a hyperparameter that allows to adjust the confidence interval manually~\cite{xia_finite_2017}. 
However, setting such a hyperparameter is difficult in practice since it depends on the unknown mean and variance of the reward and cost distributions. To address issue~(i3), current approaches set the UCB of the ratio to infinity~\cite{xia_finite_2017} or the cost LCB to a small positive value~\cite{xia_budgeted_2015}. These heuristic solutions largely ignore the information already acquired about the cost distribution and tend to cause either (i1) or (i2).

\begin{figure}[htb]
     \centering
         \includegraphics[width=\linewidth]{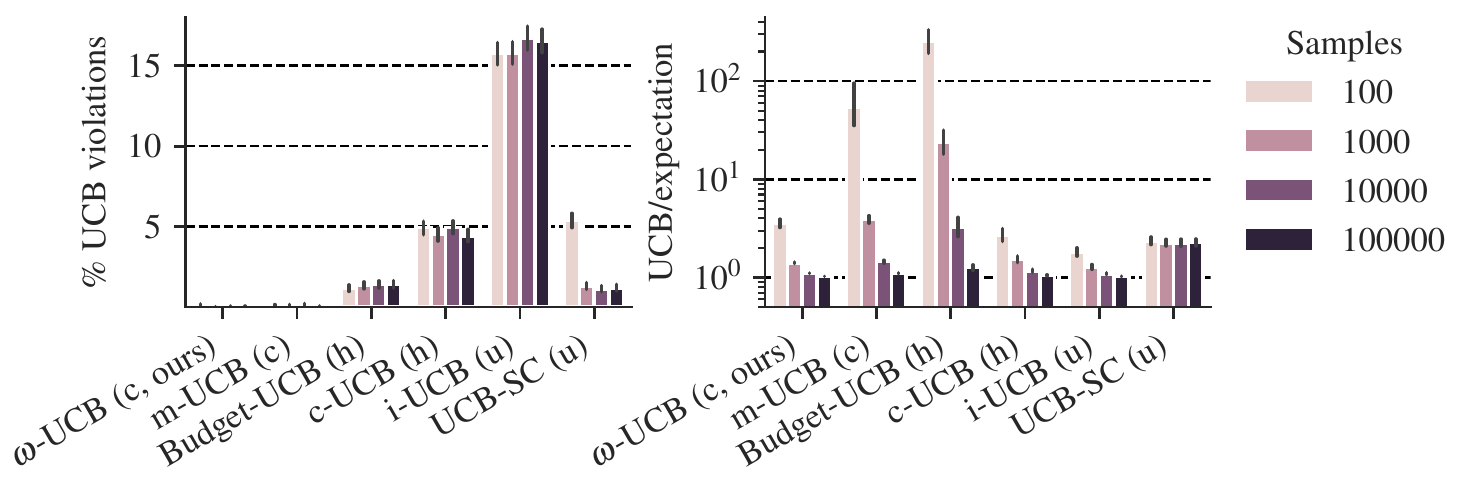}
         
    \caption{Issues of existing work}
    \label{fig:issues-illustration}
\end{figure} 

\paragraph{Contributions} 
(1) We derive asymmetric confidence intervals for bounded random variables. These intervals have the same range as the random variable and scale with the distance between the sample mean and the boundaries. 
Our formula generalizes Wilson's score interval for binomial proportions~\cite{wilson_probable_1927} to arbitrary bounded random variables.
(2) We introduce a policy called $\omega$-UCB, which leverages these confidence intervals to address issues (i1)--(i3). 
We also propose an extension of $\omega$-UCB, called $\omega^*$-UCB, that uses the observed sample variances of the arms' rewards and cost to further tighten the UCB. (3) We prove that our policy has logarithmic regret. (4) We conduct experiments on typical settings found in the literature and real-world social network advertising data to compare the performance of $\omega$-UCB and $\omega^*$-UCB against state-of-the-art policies. 
Our results demonstrate that both policies have substantially lower regret than the competitors for both small and large budgets.
(5) We share the code of our experiments.\footnote{
\ifpreprint
\url{https://github.com/heymarco/OmegaUCB}
\else
\url{https://anonymous.4open.science/r/OmegaUCB}
\fi
}

%% file: problem_definition.tex
\section{Problem definition}
\label{sec:problem_definition}
We focus on a stochastic setting with $K$ arms.
Each arm $k$ has continuous or discrete reward and cost distributions with unknown expected values $\expectrew \in [0,1)$ and $\expectcost \in (0,1]$, respectively. Assume without loss of generality that arm $k=1$ has the highest ratio ${\expectrew}/{\expectcost}$ among all arms.
At time $t$ a player chooses an arm $k_t\in\{1,\dots,K\}$ and observes the reward $r_t\in[0,1]$ and the cost $c_t\in[0,1]$. 
We assume that the arms are independent and that rewards and costs observed at different time steps are independent and identically distributed (iid). This is consistent with previous work~\cite{xia_budgeted_2015, xia_thompson_2015, xia_finite_2017, watanabe_kl-ucb-based_2017, watanabe_ucb-sc_2018}. 
We do not make any assumptions about the correlation between rewards and costs of the same arm. 
The game ends after $T_B$ plays that exhaust the available budget $B$. 

Let $\mathds{1}_k(k_t)$ be the indicator function: $\mathds{1}_k(k_t) = 1$ iff $k_t = k$, else $\mathds{1}_k(k_t)=0$. The number of plays, and the sample average of rewards and costs of arm $k$ at time $T$ are:
\begin{equation}
    n_k(T)=\sum_{t=1}^{T}\mathds{1}_k(k_t)\qquad 
    \bar\mu_k^r(T) = \frac{1}{n_k(T)}\sum_{t=1}^T \mathds{1}_k(k_t)r_t \qquad \bar\mu_k^c(T) = \frac{1}{n_k(T)}\sum_{t=1}^T \mathds{1}_k(k_t)c_t
\end{equation}

The goal of the player is to minimize the pseudo-regret compared to the cumulative reward $R^*$ of an optimal policy, given by
$R^* - \mathbb{E}\sum_{t=1}^{T_B}r_t$. Finding such optimal policy in Budgeted MABs is known to be np-hard, due to the ``knapsack problem''~\cite{tran-thanh_epsilon-first_2010}. However, choosing arm~1 only leads to a suboptimality of at most $2\mu_1^r/\mu_1^c$, negligible for not too small budgets~\cite{xia_thompson_2015}. Thus, previous work~\cite{xia_budgeted_2015, xia_thompson_2015, xia_finite_2017, watanabe_kl-ucb-based_2017, watanabe_ucb-sc_2018}, as well as our own approach, aim to minimize regret relative to a policy that always selects arm~1:
\begin{equation}
	\text{Regret} = \sum_{i=1}^K \expectcost\Delta_k\mathbb{E}[n_k(T_B)],\qquad 
	\textrm{where~} \Delta_k = \frac{\mu_1^r}{\mu_1^c} - \frac{\expectrew}{\expectcost}
\end{equation}

%% file: related_work.tex
\section{Related work}\label{sec:related_work}

There exists a plethora of different MAB-related settings and policies. We refer to~\cite{lattimore_BanditAlgorithms_2020} for an overview and focus on algorithms developed for the Budgeted MAB setting in this section.

Tran-Thanh et al.~\cite{tran-thanh_epsilon-first_2010} introduced the Budgeted MAB problem and proposed an $\epsilon$-first policy. Subsequent policies KUBE~\cite{tran-thanh_knapsack_2012} and PD-BwK~\cite{badanidiyuru_bandits_2013} address budgeted MABs as Bandits with Knapsacks, where the size of the knapsack represents the available budget. However, KUBE assumes deterministic costs and both methods require knowledge of $B$. Another approach, UCB-BV1~\cite{ding_multi-armed_2013}, addresses discrete random costs. Such assumptions limit the applicability of KUBE, PD-BwK, and UCB-BV1. Later solutions~\cite{xia_budgeted_2015, xia_thompson_2015, xia_finite_2017, watanabe_kl-ucb-based_2017, watanabe_ucb-sc_2018} adapted concepts from traditional MABs, such as Upper Confidence Bound (UCB)~\cite{auer_finite-time_2002} or Thompson sampling~\cite{thompson_likelihood_1933,thompson_theory_1935} and can deal with continuous random costs and unknown budget, improving their applicability. 
However, the one policy based on Thompson sampling, BTS~\cite{xia_thompson_2015}, requires transforming continuous rewards and costs into Bernoulli-samples. As a result, the policy disregards information about the variance of rewards and costs, causing over-pessimism (i2) when the variance of the reward or cost distribution is small. MRCB~\cite{xia_budgeted_2015} deals with the challenge of playing multiple arms in each time step; when playing only one arm at a time, the policy becomes m-UCB~\cite{xia_finite_2017} that is similar to our policy. However, m-UCB relies on Hoeffding's inequality, which does not take the distance between a random variable's sample mean and boundaries into account. To see why this is problematic, consider the following example:

\begin{example}[m-UCB]
    Assume two arms with $\mu_1^r = 0.8, \mu_1^c = 0.2$, $\mu_2^r = 0.1, \mu_2^c = 0.1$. Clearly, arm 1 should be preferred for a reasonably large budget. However, m-UCB shows a consistent bias towards pulling arm 2. For instance, if $t=10000$ and $n_1=n_2=1000$, using m-UCB with $\alpha=1$ would yield reward-cost UCB values\footnote{See Appendix~\ref{app:related_work} for mathematical details} of $\approx2.95$ for arm~1 and $\approx48.6$ for arm~2 due to the high influence of the denominator in \eq{eq:high-level-ucb} (rhs). m-UCB would hence pull arm 2. In comparison, $\omega$-UCB would compute values of $5.5$ and $2.1$, and pull arm 1.
\end{example}

All the above policies either 
have issues (i1)--(i3)~\cite{xia_finite_2017,xia_budgeted_2015,watanabe_ucb-sc_2018,watanabe_kl-ucb-based_2017,xia_thompson_2015}, 
are not designed for continuous random costs~\cite{tran-thanh_knapsack_2012,ding_multi-armed_2013,xia_thompson_2015}, 
or have been shown to perform inferior to others~\cite{tran-thanh_epsilon-first_2010,ding_multi-armed_2013,tran-thanh_knapsack_2012}.
Figure \ref{fig:related_work} provides a compilation of existing head-to-head empirical comparisons between various policies. Upwards pointing triangles indicate that the policy in the corresponding row outperformed the policy in the corresponding column in the respective paper, while downward pointing triangles indicate the opposite. Circles represent cases where both policies performed similarly, while horizontal lines indicate that the policies have not been compared.
One sees that KUBE~\cite{tran-thanh_knapsack_2012} outperforms $\epsilon$-first~\cite{tran-thanh_epsilon-first_2010}, while UCB-BV1~\cite{ding_multi-armed_2013} and BTS\cite{xia_budgeted_2015} outperform KUBE. 
UCB-BV1 is inferior to more recent policies~\cite{xia_thompson_2015, xia_budgeted_2015, watanabe_kl-ucb-based_2017, watanabe_ucb-sc_2018}. BTS~\cite{xia_thompson_2015}, b-greedy~\cite{xia_finite_2017}, and \{i, c, m\}-UCB~\cite{xia_finite_2017} outperform PD-BwK~\cite{badanidiyuru_bandits_2013}. 
We will compare our policy to the best performing existing policies~--- BTS, Budget-UCB, \{i, c, m\}-UCB, b-greedy, and UCB-SC+. 

\begin{figure}[t]
    \begin{subfigure}[b]{.282\textwidth}
    \raggedleft
        \begin{tabular}{rr}
             \toprule
             Policy& Ref. \\
             \midrule
             $\varepsilon$-first & \cite{tran-thanh_epsilon-first_2010} \\ 
             KUBE &\cite{tran-thanh_knapsack_2012} \\
             UCB-BV1 & \cite{ding_multi-armed_2013}\\ 
             PD-BwK & \cite{badanidiyuru_bandits_2013} \\
             Budget-UCB & \cite{xia_budgeted_2015} \\
             BTS & \cite{xia_thompson_2015} \\
             MRCB & \cite{xia_budgeted_2016} \\
             m-UCB & \cite{xia_finite_2017}\\
             b-greedy & \cite{xia_finite_2017}\\
             c-UCB & \cite{xia_finite_2017}\\
             i-UCB & \cite{xia_finite_2017} \\
             KL-UCB-SC+ & \cite{watanabe_kl-ucb-based_2017} \\
             UCB-SC+ & \cite{watanabe_ucb-sc_2018} \\
             \bottomrule
        \end{tabular}
        \vspace{0cm}
    \end{subfigure}
    \begin{subfigure}[b]{.38\textwidth}
        \raggedright
    	\includegraphics[width=\linewidth]{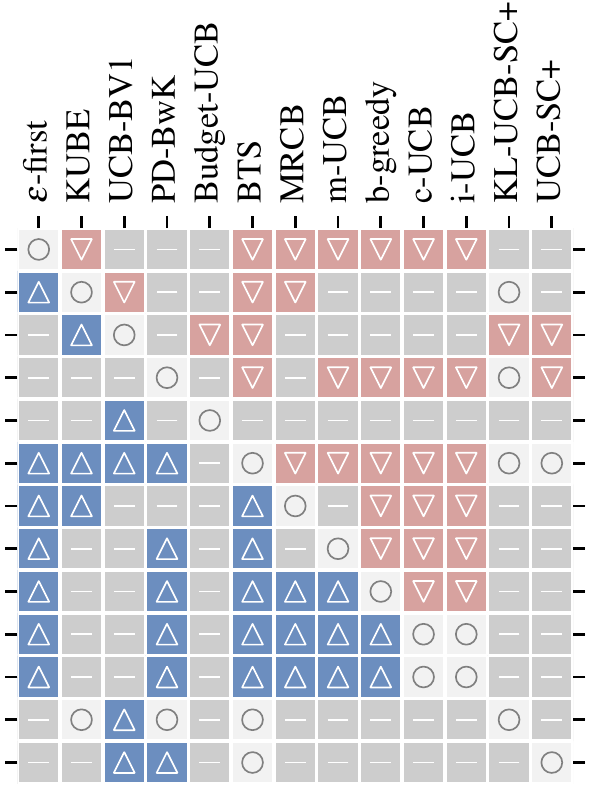}
        \vspace{-.287cm}
    \end{subfigure}
    \begin{subfigure}[b]{.2\textwidth}
        \begin{tabular}{lcc}
            \toprule
            Year& Type & Compared \\
            \midrule 
             2010& -- & $\cross$ \\
             2012 & -- & $\cross$\\
             2013 & h & $\cross$\\
             2013 & c & $\cross$\\
             2015 & h & $\checkmark$\\
             2015 & -- & $\checkmark$\\
             2016 & c & -- \\
             2017 & c & $\checkmark$ \\
             2017 & -- & $\checkmark$\\
             2017 & h  & $\checkmark$\\
             2017 & u  & $\checkmark$\\
             2017 & u  & --\\
             2018 & u  & $\checkmark$\\
             \bottomrule
        \end{tabular}
        \vspace{0cm}
    \end{subfigure}
    
\caption{Empirical performance of different Budgeted MAB policies according to related work}
\label{fig:related_work}
\end{figure}

%% file: our_approach.tex
\section{Our policy}

We now detail our policy $\omega$-UCB and analyze it theoretically.

\subsection{\texorpdfstring{$\omega$}{omega}-UCB}\label{sec:omega-ucb}

$\omega$-UCB starts by playing each arm once. 
At each subsequent time step $t$, the policy chooses the arm $k_t$ with the highest upper confidence bound of the ratio of the expected reward $\mu^r_k$ to the expected cost $\mu_k^c$. 
Let $\omega_{k+}^r(\alpha,t)$ denote the upper confidence bound of $\mu^r_k$ for a confidence level $1-\alpha$. Similarly, $\omega_{k-}^c(\alpha,t)$ is the lower confidence bound of $\mu_k^c$. 
$\omega$-UCB chooses $k_t$ according to:
\begin{equation}\label{eq:index}
    k_t = \argmax_{k\in[K]} \Omega_k(\alpha,t),\qquad \text{where } \Omega_k(\alpha,t) = \frac{\omega_{k+}^r(\alpha,t)}{\omega_{k-}^c(\alpha,t)}
\end{equation} 

Unlike other policies that rely on the same principle~\cite{xia_finite_2017,xia_budgeted_2015,xia_budgeted_2016,badanidiyuru_bandits_2013}, $\omega$-UCB calculates asymmetric confidence bounds that are shifted towards the center of the range of the random variable. This leads to tighter UCB for the reward-cost ratio especially when an arm's expected cost or the number of plays is low.

\begin{restatable}[Asymmetric confidence interval for bounded random variables]{theorem}{thwilsonscore}
\label{th:wilson-score}
Let $X$ be a random variable bounded in the interval $[m,M]$, with unknown expected value $\mu\in[m,M]$ and variance $\sigma^2$. 
Let $z$ denote the number of standard deviations required to achieve $1-\alpha$ confidence in coverage of the standard normal distribution. 
Let $\bar\mu$ be the sample mean of $n$ iid samples of $X$. Then,
\begin{equation}\label{eq:wilson}
    \Pr[\mu\not\in \left[\omega_-(\alpha),\omega_+(\alpha)\right]] \leq \alpha,
\quad\textrm{with~}
\omega_\pm(\alpha)
= \frac{B}{2A} \pm \sqrt{\frac{B^2}{4A^2} - \frac{C}{A}},
\end{equation}
where 
\begin{equation}
A=n+z^2\eta, \quad B=2n\bar\mu + z^2\eta (M + m), \quad C=n\bar\mu^2 + z^2\eta Mm, \quad\textrm{and}
\end{equation}
\begin{equation}
	\eta=\frac{\sigma^2}{(M-\mu)(\mu-m)} \textrm{ if } \mu\in (m, M),\quad \textrm{ and } \eta=1 \textrm{ if } \mu\in\{m, M\}.
\end{equation}
\end{restatable}

\begin{proof}
    See Appendix~\ref{proof:wilson-score}.
\end{proof} 

Appendix~\ref{app:asymmetry} illustrates the level of asymmetry of the above interval. In summary, the interval is shifted towards the center of the range of the random variable and is maximal for $\mu\in \{m, M\}$. However, asymmetry decreases as $n$ increases. 

According to the Bhatia-Davis inequality~\cite{bhatia_BetterBoundVariance_2000}, $0\leq\sigma^2\leq(M-\mu)(\mu-m)$, and hence $\eta\in[0,1]$. 
For the special case of Bernoulli random variables, $\eta=1$, $m=0$, $M=1$, and Theorem~\ref{th:wilson-score} recovers Wilson's original confidence interval for Binomial proportions~\cite{wilson_probable_1927}. 
However, this theorem is more flexible than Wilson's original method. It enables tighter confidence intervals for non-Bernoulli costs or rewards, by setting $\eta<1$ when a variance estimate is available. 
Our experiments will demonstrate that this flexibility leads to a significant improvement in performance.

The following theorem connects $\Omega_k(\alpha,t)$ and the confidence level of the ratio ${\expectrew}/{\expectcost}$.

\begin{restatable}[UCB for ratio of expected values of two random variables]{theorem}{cordeviationrewardcost}
\label{cor:deviation-reward-cost}
Let $R$ and $C$ be two bounded random variables with expected values $\mu^r \geq 0$ and $\mu^c > 0$. 
Let $\omega_+^r(\alpha)\geq 0$ denote the upper confidence bound of $R$ and $\omega_-^c(\alpha) > 0$ the lower confidence bound of $C$ as given in Theorem~\ref{th:wilson-score}.
Let $\Omega(\alpha) = {\omega_+^r(\alpha)}/{\omega_-^c(\alpha)}$.
Then,
\begin{equation}
    \Pr[\frac{\mu^r}{\mu^c} > \Omega(\alpha)] \leq \alpha
\end{equation}
\end{restatable}

\begin{proof}
Define events $E1 = \mu^r > \omega_+^r(\alpha)$ and $E2 = \mu^c < \omega_-^c(\alpha)$.
A violation of the UCB of the reward-cost ratio requires that either $E1$ or $E2$ occurs, or that both events happen simultaneously. Therefore, by the union bound we have that $\Pr\left[{\mu^r}/{\mu^c} > \Omega_k\right] = \Pr\left[{\mu^r}/{\mu^c} > {\omega_+^r(\alpha)}/{\omega_-^c(\alpha)}\right] \leq \Pr[E1] + \Pr[E2]$. $E1$ and $E2$ both occur with probablity $\leq\alpha/2$, hence $\Pr\left[{\mu^r}/{\mu^c} > \Omega_k\right] \leq \alpha$.
\end{proof}

A UCB-sampling policy that keeps the parameter $\alpha$ constant leads to linear regret in the worst case.
This is because such a policy will eventually stop exploring arms that may have high costs and low rewards in the beginning. 
To avoid this problem, $\omega$-UCB decreases the value of $\alpha$ as the time~$t$ progresses, similarly to the UCB1-policy~\cite{auer_finite-time_2002}. 
This adaptive approach helps to ensure continued exploration of arms and guarantees sub-linear regret. 
The following theorem introduces the scaling law and relates it to the confidence level. 

\begin{restatable}[Time-adaptive confidence interval]{theorem}{thadaptiveconf}
\label{th:adaptive-conf}
For an arm $k$, let $\expectrew$ be its expected reward, $\expectcost$ its expected cost, and $\Omega_k(\alpha,t)$ the upper confidence bound for ${\expectrew}/{\expectcost}$, as in \eq{eq:index}. For $\rho, t >0 $, and $\alpha(t) < 1- \sqrt{1-t^{-\rho}}$ it holds that
\begin{equation}
    \Pr[\Omega_k(\alpha,t) \geq \frac{\expectrew}{\expectcost}] \geq 1 - \alpha(t),
\end{equation}
that is, the upper confidence bound holds asymptotically almost surely.
\end{restatable}

\begin{proof}
    See Appendix~\ref{proof:theorem-adaptive-conf}.
\end{proof}

With $\alpha(t) < 1 - \sqrt{1-t^{-\rho}}$, the confidence level approaches $1$ as time $t$ goes to infinity. This encourages exploration of arms that are played less frequently. Moreover, it establishes a logarithmic dependence between $z$ in Theorem~\ref{th:wilson-score} and $t$, i.e., $z_\rho(t)=\sqrt{2\rho\log t}$, which will be useful in our regret analysis. The next section analyzes the worst-case regret of $\omega$-UCB. To simplify notation, we abbreviate $\Omega(\alpha, t)$ as $\Omega(t)$, $\omega_{k-}^c(\alpha, t)$ as $\omega_{k-}^c(t)$, and $\omega_{k+}^r(\alpha, t)$ as $\omega_{k+}^r(t)$ in our regret analysis. 

%% file: theoretical_analysis.tex
\subsection{Regret analysis}

In this section, we bound the expected number of suboptimal plays $\mathbb{E}[n_k(\tau)]$ before some time step $\tau$ and use this analysis to derive the regret bound of $\omega$-UCB (cf.~Theorem~\ref{th:worst-case-regret}). 

\begin{restatable}[Number of suboptimal plays]{theorem}{thsuboptimalplays}
	\label{th:suboptimal-plays}
	For $\omega$-UCB, the expected number of plays of a suboptimal arm $k>1$ before time step $\tau$, $\mathbb{E}[n_k(\tau)]$, is upper-bounded by
	\begin{equation}\label{eq:result-n-suboptimal-arms}
		\begin{split}
			\mathbb{E}[n_k(\tau)] \leq 1 + n_k^*(\tau) + \xi(\tau, \rho),
		\end{split}
	\end{equation}
	where 
	\begin{equation}
		\xi(\tau, \rho) = \left(\tau-K\right)\left(2-\sqrt{1-\tau^{-\rho}}\right) - \sumtau\sqrt{1-t^{-\rho}},
	\end{equation}
    \begin{equation}
        n_k^*(\tau) = \frac{8\rho\log \tau}{\delta_k^2}\max\left\{ \frac{\eta_k^r\expectrew}{1-\expectrew}, \frac{\eta_k^c(1-\expectcost)}{\expectcost} \right\},\quad \delta_k = \frac{\Delta_k}{\Delta_k + \frac{1}{\expectcost}},
    \end{equation} 
and $K$ and $\Delta_k$ are defined as before, cf.~Section~\ref{sec:problem_definition}. 
\end{restatable}

\begin{proof}
	Appendix~\ref{proof:th_suboptimal-plays} contains the proof of the theorem. The derivation of $\delta_k$ is in Appendix~\ref{proof:delta-gap}.
\end{proof}

Lemma~4 in~\cite{xia_finite_2017} allows us to derive worst-case regret of $\omega$-UCB from Theorem~\ref{th:suboptimal-plays}.

\begin{restatable}[Worst-case regret]{theorem}{thworstcaseregret}
\label{th:worst-case-regret}
Define $\tau_B = \left\lfloor{2B}/{\min_{k\in[K]}\expectcost}\right\rfloor$ and $\Delta_k, n_k^*(\tau_B)$, and $\xi(\tau_B,\rho)$ as before.
For any $\rho>0$, the regret of $\omega$-UCB is upper-bounded by  
\begin{equation}
\label{eq:regret-wucb}
    \text{Regret} \leq \sum_{k=2}^K\Delta_k\left(1 + n_k^*(\tau_B) + \xi(\tau_B,\rho)\right)
    + \mathcal{X}(B)\sum_{k=2}^K\Delta_k + \frac{2\mu_1^r}{\mu_1^c},
\end{equation}
where $\mathcal{X}(B) $ is in $ \mathcal{O}\left(\frac{B}{\mu_{min}^c}e^{-0.5B\mu_{min}^c}\right)$.
\end{restatable}

\begin{proof}
Lemma~4 of~\cite{xia_finite_2017} provides a policy-independent regret expression for Budgeted MAB policies: 
\begin{equation}\label{eq:lemma4-xia}
    \text{Regret} \leq \sum_{k=2}^K\Delta_k \mathbb{E}[n_k(\tau_B)] + \mathcal{X}(B)\sum_{k=2}^K\Delta_k + \frac{2\mu^r_1}{\mu_1^c},\quad \tau_B = \left\lfloor\frac{2B}{\min_{k\in[K]}\expectcost}\right\rfloor
\end{equation}
Substituting $\mathbb{E}[n_k(\tau_B)]$ in \eq{eq:lemma4-xia} with the result from Theorem~\ref{th:suboptimal-plays} completes the proof.
\end{proof}


The term $\xi(\tau, \rho)$ decreases, while $n_k^*(\tau)$ increases with $\rho$.
Further derivations show that for increasingly large budgets, $\xi(\tau,\rho)$ converges for $\rho>1$, grows logarithmic for $\rho=1$, and diverges on the order of $\mathcal{O}(B^{1-\rho})$ for $\rho < 1$; see Appendix~\ref{proof:worst-case-regret-complex} for the details. This results in the following asymptotic behavior:

\begin{restatable}[Asymptotic regret]{theorem}{thworstcaseregretcomplex}
\label{th:worst-case-regret-complex}
The regret of $\omega$-UCB is in
\begin{equation}
    \mathcal{O}\left(B^{1-\rho}\right) \textrm{ for } 0 < \rho < 1, \quad \textrm{and in}\quad \mathcal{O}(\log B) \textrm{ for } \rho \geq 1.
\end{equation}
\end{restatable}

\begin{proof}
    See Appendix~\ref{proof:worst-case-regret-complex}.
\end{proof}

%% file: experiments.tex
\section{Experimental Setup}
This section presents the experimental setup used to evaluate the policies. 
We introduce the MAB settings, followed by the configurations of $\omega$-UCB and its competitors. 
We conducted the experiments on a server with 32 cores, each running at 2.0 GHz, and 128 GB of RAM. 

\subsection{Budgeted MAB settings}

We use MAB settings based on synthetic and real data, which we describe separately.
Each setting comprises a specific combination of reward and cost distributions, and the number of arms $K$. 
See Table~\ref{tab:evaluation_settings} for a summary of the evaluation settings. 
  
\paragraph{Synthetic Data.} 
Previous studies on Budgeted Multi-Armed Bandits (MABs) have used synthetic settings with rewards and costs drawn from discrete (Bernoulli or Generalized Bernoulli with possible outcomes $\{0.0, 0.25, 0.5, 0.75, 1.0\}$) or continuous (Beta) distributions~\cite{xia_budgeted_2015,xia_finite_2017,watanabe_ucb-sc_2018,watanabe_kl-ucb-based_2017}. 
These studies typically generate parameters randomly within a given range~\cite{xia_finite_2017,xia_budgeted_2015,xia_thompson_2015,xia_budgeted_2016} and use 10 to 100 arms~\cite{xia_budgeted_2015, xia_thompson_2015, xia_finite_2017, xia_budgeted_2016, watanabe_ucb-sc_2018}. We follow this approach and set the parameter ranges to those used in related work. 

\paragraph{Social-media advertisement data.}
We also evaluate our policy in a social media advertisement scenario described in Example \ref{ex:adverstising}.
We use real-world data from~\cite{lemsalu_facebook_2017}. 
It contains information about different ads based on their target gender (female or male) and age category (30--34, 34--39, 40--44, 45--49), along with the number of displays and clicks, the total cost, and the number of purchases. 
We group the ads by target gender and age category, resulting in 19 ``advertisement campaigns'' (Budgeted MAB settings). Each campaign has between 2 and 93 ads (arms). We compute the expected rewards $\expectrew$ and costs $\expectcost$ of each ad as the average revenue per click and average cost per click, respectively. We model both discrete and continuous rewards and costs. 
For the discrete case, we sample rewards and costs from two Bernoulli distributions with expected values of $\expectrew$ and $\expectcost$, respectively. 
In the continuous case, we use a Beta distribution and sample the distribution parameters from a uniform distribution with a range of (0, 5). 
We then adjust one of the parameters to ensure that the expected values of rewards and costs match $\expectrew$ and $\expectcost$.

\begin{table}[htb]
\centering

\caption{Our evaluation settings}

\label{tab:evaluation_settings}
    \begin{tabular}{llcrll}
        \toprule
        Type & Distribution & Parameters & K & Used in & Id\\
        \midrule
        \multirow{9}{*}{Synthetic} &\multirow{3}{*}{Bernoulli} & \multirow{3}{*}{$\mathcal{U}(0,1)$}& 10& \cite{xia_thompson_2015, xia_finite_2017} & S-Br-10\\
        && & 50& \cite{xia_finite_2017} & S-Br-50\\
        && & 100& \cite{xia_budgeted_2015, xia_thompson_2015} & S-Br-100\\
        \cmidrule{2-6}
        &\multirow{3}{*}{\shortstack[l]{Generalized\\ Bernoulli}} & \multirow{3}{*}{$\mathcal{U}(0,1)$} & 10 & \cite{xia_thompson_2015,xia_budgeted_2016} & S-GBr-10\\
        && & 50 & \cite{xia_budgeted_2016} & S-GBr-50\\
        && & 100 & \cite{xia_thompson_2015} & S-GBr-100\\
        \cmidrule{2-6}
        &\multirow{3}{*}{Beta} & \multirow{3}{*}{$\mathcal{U}(0,5)$} & 10 & \cite{xia_finite_2017,xia_budgeted_2016} & S-Bt-10\\
        && & 50 & \cite{xia_finite_2017,xia_budgeted_2016} & S-Bt-50\\
        && & 100 & \cite{xia_budgeted_2015} & S-Bt-100\\
        \midrule
        \multirow{2}{*}{Facebook} & Bernoulli & given & $[2, 97]$& -- & FB-Br \\
        \cmidrule{2-6}
        & Beta & randomized & $[2, 97]$ & -- & FB-Bt\\
        \bottomrule
    \end{tabular}
\end{table}

\subsection{Budgeted MAB policies}
We test the performance of two variants of our policy: $\omega$-UCB and $\omega^*$-UCB.
The $\omega$-UCB variant uses a fixed value of $\eta=1$ for rewards and costs. 
The $\omega^*$-UCB uses $\eta_k^r=\eta_k^c=1$ as default but estimates their values from data once arm $k$ has been played sufficiently many times ($n_k(T) \geq 30$),
\begin{equation}
    \bar\eta_k=\frac{\bar\sigma_k^2}{(M-\bar\mu_k)(\bar\mu_k-m)},
\end{equation}
where bars refer to sample estimates as before.
We experiment with two values of the hyperparameter~$\rho$: $\rho=1$ and $\rho={1}/{4}$. The former is the minimum value for which we have proven logarithmic regret. The latter has performed well in our sensitivity study, as we will demonstrate in Section~\ref{sec:sensitivity-study}.

We compare the performance of our policy to several other state-of-the-art Budgeted MAB policies, including BTS, Budget-UCB, i-UCB, c-UCB, m-UCB, b-greedy, and UCB-SC+. 
We set the hyperparameters for each competitor to the values recommended in their respective papers. 
Appendix~\ref{app:related_work} features details about the policies and their hyperparameters.

\section{Results}

To observe the asymptotic behavior of the policies, we set the budget $B$ to $1.5\cdot 10^5$ times the minimum expected cost. We execute each policy until the available budget is depleted and report the average results over 100 independent repetitions with the repetition index as the seed for the random number generator. Since in each repetition of the experiment, we draw the expected cost $\expectcost$ and expected reward $\expectrew$ uniformly at random, we normalize the budget in our graphs. 
Section~\ref{sec:evaluation_competitors} compares the performance of our policy and its competitors. Section~\ref{sec:sensitivity-study} examines the sensitivity of $\omega$-UCB and $\omega^*$-UCB to the hyperparameter $\rho$. Note that some graphs (the ones in Figure~\ref{fig:regret_evaluation}) omit confidence intervals for better accessibility. Graphs with 95\% confidence intervals are available in Appendix~\ref{sec:evaluation_results_ci}.

\subsection{Performance of Budgeted MAB policies}\label{sec:evaluation_competitors}

\paragraph{Synthetic Bernoulli.} 
Figure~\ref{fig:synth_bernoulli} shows the regret of the policies for Bernoulli-distributed rewards and costs. Blue lines display results for $\omega$-UCB, while grey lines stand for other policies. We do not present $\omega^*$-UCB in this experiment since its results are almost identical to $\omega$-UCB.
$\omega$-UCB and BTS achieve logarithmic regret and demonstrate better asymptotic behavior than other methods. Although m-UCB, c-UCB, and i-UCB may outperform $\omega$-UCB with $\rho=1$ for small budgets, their regret grows rapidly as the budget increases, indicating poor asymptotic behavior. For $K=50$ and $K=100$, our policy has lower regret than BTS. BTS outperforms $\omega$-UCB with $\rho=1$ only on the 10-armed bandit. $\omega$-UCB with $\rho={1}/{4}$ outperforms all other policies on small and large budgets and regardless of $K$. Comparing $\rho=1/4$ with $\rho=1$, One can see that the curve for $\rho=1$ is linear (the x-axis is logarithmic), while the curve for $\rho=1/4$ is convex. We conclude that $\rho=1/4$ leads to smaller regret than 
$\rho=1$ for not too large budgets but that 
$\rho=1$ performs better asymptotically.

\paragraph{Synthetic Generalized Bernoulli and synthetic Beta.} 
Figure~\ref{fig:synth_multinomial} shows the regret of the policies for rewards and costs drawn from Generalized Bernoulli distributions and Figure~\ref{fig:synth_beta} for Beta distributed rewards and costs. 
Besides $\omega$-UCB, we also present the results for $\omega^*$-UCB which estimates $\eta_k^r$ and $\eta_k^c$ from the observed sample variance of rewards and costs.
$\omega$-UCB and $\omega^*$-UCB with $\rho={1}/{4}$ outperform their competitors, except for $K=100$ in the Beta bandit where m-UCB performs comparable. We further notice that BTS is not competitive in this evaluation setting, likely because the policy cannot account for the (often very small) variance of the sampled Beta distributions. Our $\omega^*$-UCB policy is advantageous in such cases.

\begin{figure*}
    \centering
    \begin{subfigure}[b]{\textwidth}
        \centering
        \includegraphics[width=\linewidth]{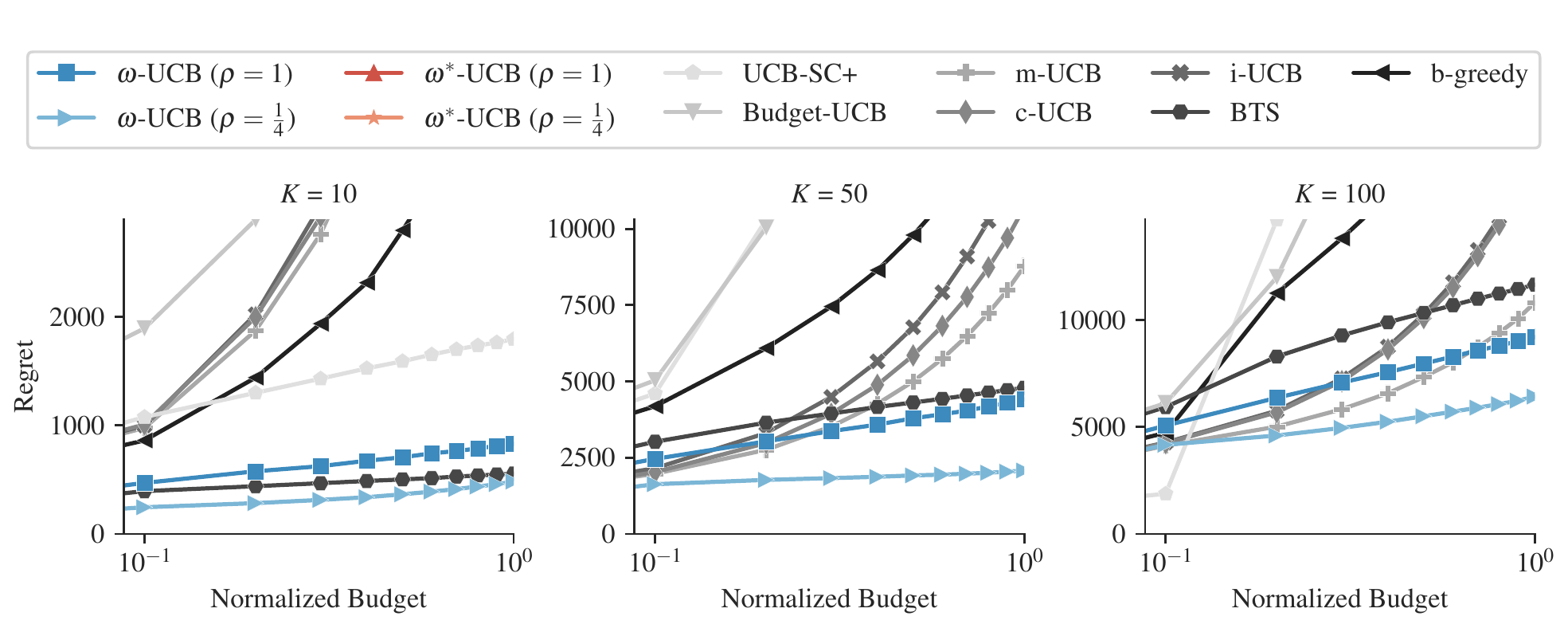}
        \caption{Settings S-Br-$\{10, 50, 100\}$}
        \label{fig:synth_bernoulli}
    \end{subfigure}
    \hfill
    \begin{subfigure}[b]{\textwidth}
        \centering
        \includegraphics[width=\linewidth]{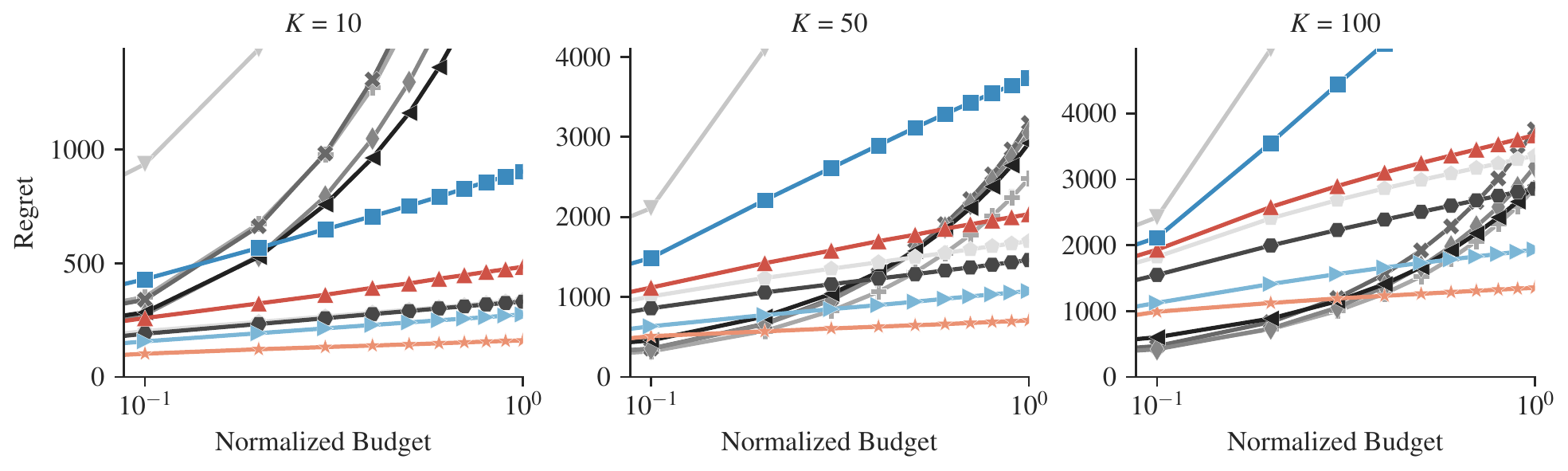}
        \caption{Settings S-GBr-$\{10, 50, 100\}$}
        \label{fig:synth_multinomial}
    \end{subfigure}
    \hfill
    \begin{subfigure}[b]{\textwidth}
        \centering
        \includegraphics[width=\linewidth]{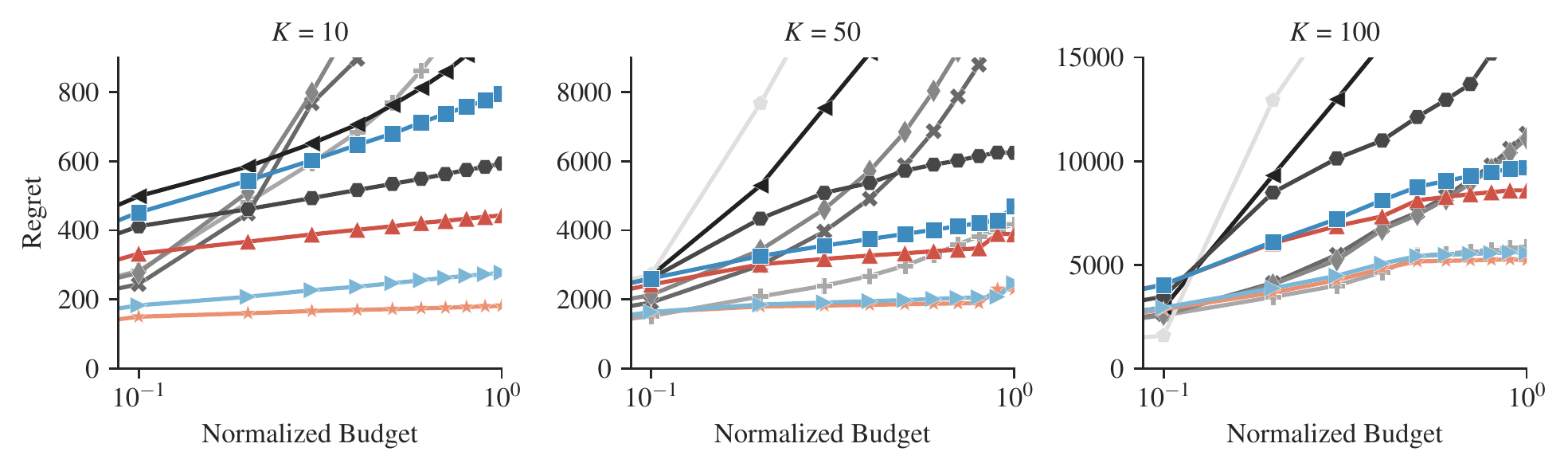}
        \caption{Settings S-Bt-$\{10, 50, 100\}$}
        \label{fig:synth_beta}
    \end{subfigure}
    \centering
    \begin{subfigure}[b]{.33\textwidth}
         \centering
        \includegraphics[width=\linewidth]{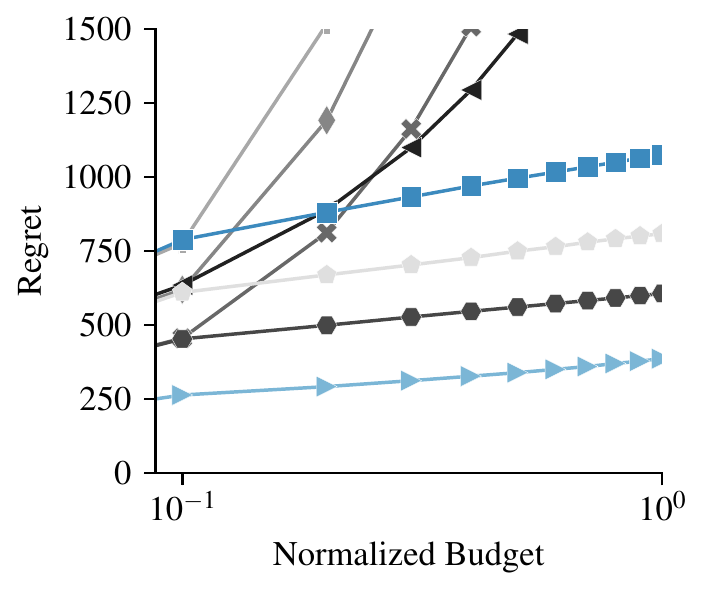}
        \caption{Setting FB-Br}
        \label{fig:facebook-bernoulli}
    \end{subfigure}
    \begin{subfigure}[b]{.33\textwidth}
         \centering
          \includegraphics[width=\linewidth]{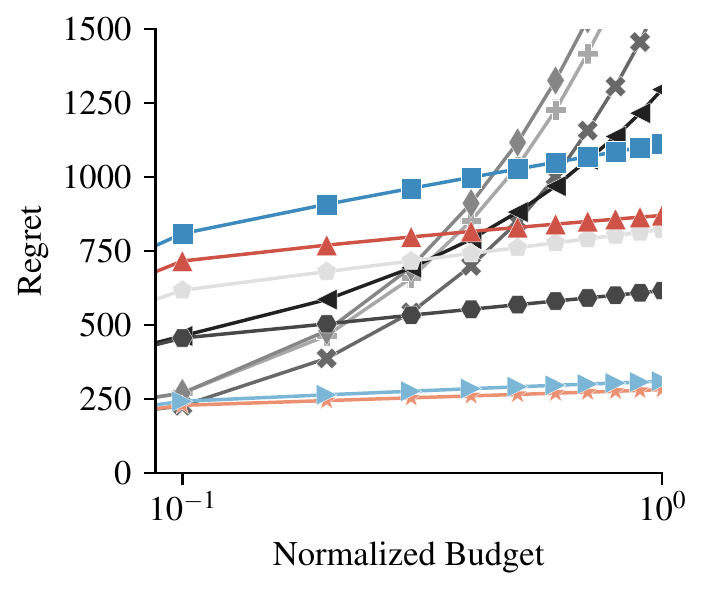}
         \caption{Setting FB-Bt}
         \label{fig:facebook-beta}
    \end{subfigure}
    \begin{subfigure}[b]{.32\textwidth}
         \centering
          \includegraphics[width=\linewidth]{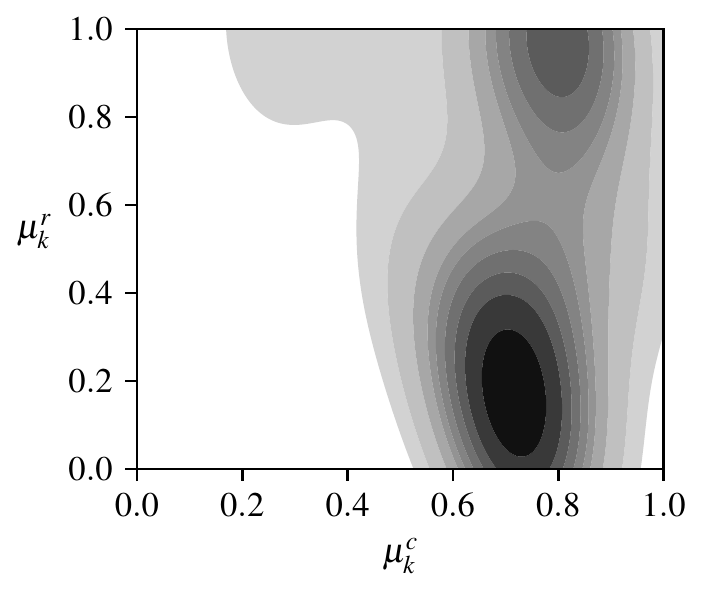}
         \caption{Example kde of $\expectrew,\expectcost$}
         \label{fig:example_kde}
    \end{subfigure}
    
    \caption{Evaluation results; refer to Figure~\ref{fig:regret_evaluation_ci} in Section~\ref{sec:evaluation_results_ci} for graphs with confidence intervals} 
\label{fig:regret_evaluation}
\end{figure*}

\paragraph{Social-media advertisement data.}

Figure~\ref{fig:facebook-bernoulli} and Figure~\ref{fig:facebook-beta} show the results of our study on social-media advertisement data~\cite{lemsalu_facebook_2017}. On this real world data set, our default choice of $\rho={1}/{4}$ outperforms all other competitors. $\rho=1/4$ also outperforms $\rho=1$ significantly, although both choices show good asymptotic behavior. Further, the advertisement data settings seem to be easier for some competitors (BTS, KL-UCB-SC+) and harder for others ({i,c,m}-UCB). 
The likely cause is that the distribution of expected rewards and costs between arms is non-uniform: costs are biased towards \numprint{1}, and rewards are biased towards the boundaries of $[0, 1]$, as the kde plot for the MAB with $K=33$ in~Figure~\ref{fig:example_kde} illustrates exemplary. Last, we observe that $\omega^*$-UCB has lower regret than $\omega$-UCB, although the effect is not as prominent as in the synthetic settings.

\subsection{Sensitivity study}\label{sec:sensitivity-study}

We investigate the performance difference between $\omega$-UCB and $\omega^*$-UCB, as well as the sensitivity of our policy with respect to the hyperparameter $\rho$. The results based on our synthetic settings (cf.~Table~\ref{tab:evaluation_settings}) are shown in Figure~\ref{fig:sensitivity_study}. $\omega$-UCB and $\omega^*$-UCB perform best with $\rho={1}/{4}$ when rewards and costs follow Bernoulli distributions. Both policies achieve comparable performance in this case. It appears that estimating $\eta$ (which is known to be 1 in Bernoulli bandits) does not result in a performance decrease of $\omega^*$-UCB compared to $\omega$-UCB.
Also, even though $\rho={1}/{8}$ works well for $\omega$-UCB when rewards and costs follow a generalized Bernoulli or Beta distribution, $\rho={1}/{4}$ remains a near-optimal choice for $\omega^*$-UCB. Based on these results, we recommend using $\omega^*$-UCB with $\rho={1}/{4}$ as a default. 

\begin{figure}
	\centering
	\includegraphics[width=\linewidth]{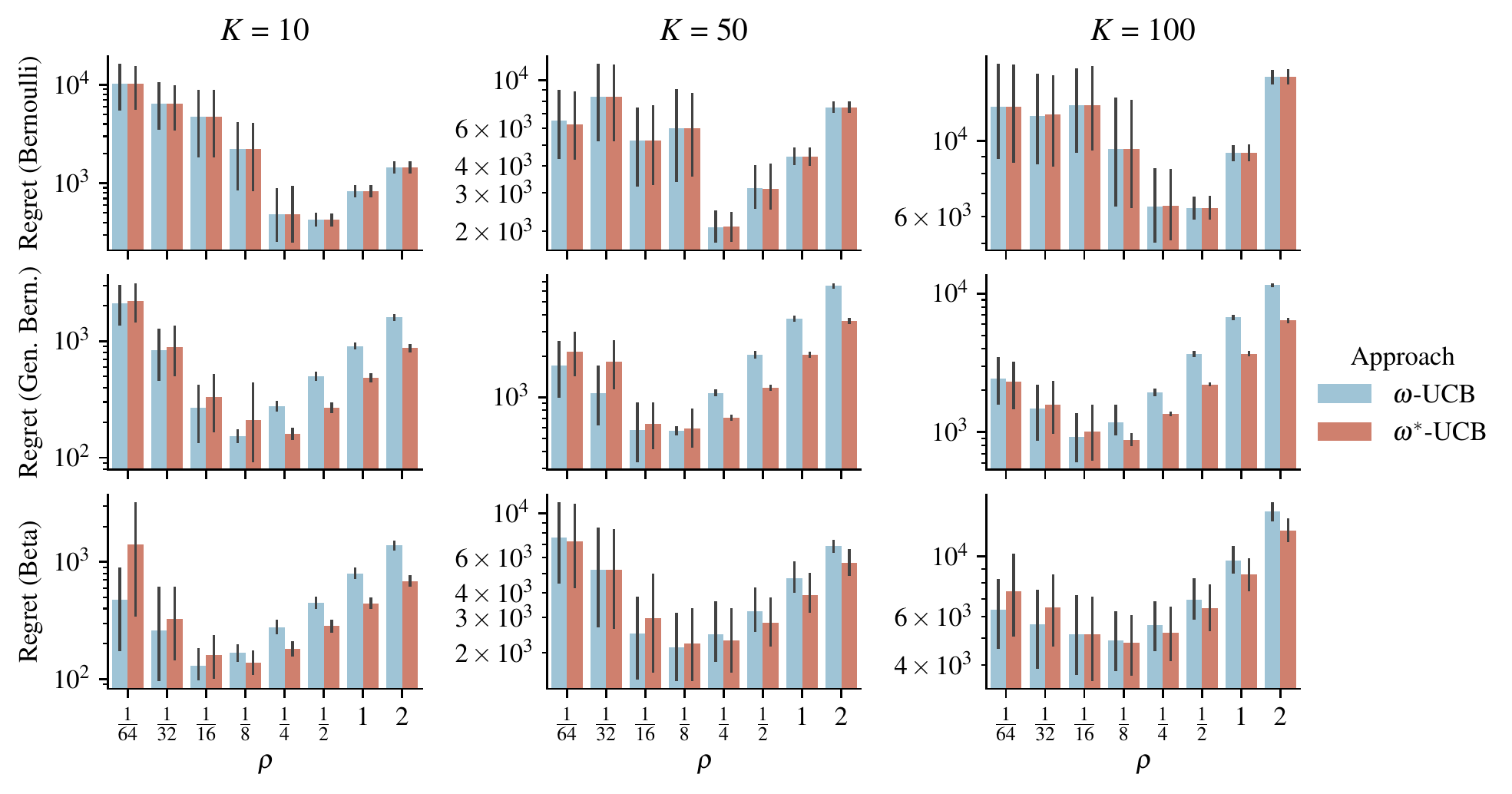}
 
	\caption{Sensitivity study showing the regret w.r.t.\ different choices of $\rho$}
	\label{fig:sensitivity_study}
\end{figure}

%% file: conclusion.tex
\section{Conclusions}

We have presented a new approach for Budgeted MABs called $\omega$-UCB. It combines UCB sampling with asymmetric confidence intervals to address issues of existing approaches. Our interval generalizes Wilson's score interval to arbitrary bounded random variables. An extension of our approach, $\omega^*$-UCB, tracks the variance of the reward and cost distributions on the fly to tighten the confidence intervals. This leads to even better performance when rewards or costs are continuous. Our analysis shows that $\omega$-UCB achieves logarithmic regret for $\rho \geq 1$, while $\rho = {1}/{4}$ performed best in our experiments. In the future, we plan to extend our approach to the non-stationary setting where the reward and cost distributions change over time. This is particularly relevant in scenarios like online advertising, where companies want to promote their products and services continuously.

%% file: appendix.tex
\appendix
\input{proofs}

\section{Summary of related approaches}\label{app:related_work}

Table~\ref{tab:competitors_indexes} summarizes our direct competitors and how they compute $\Omega_k(t)$.

\begin{table*}[htb]
    \caption{Overview of our competitors}
    \label{tab:competitors_indexes}
    \centering
    \begin{tabular}{lcl}
        \toprule
         Policy& $\Omega_k(t)$ & Comment \\
         \midrule
         \multirow{6}{*}{BTS~\cite{xia_thompson_2015}}& \multirow{6}{*}{$\frac{\textrm{sample from } \mathrm{Beta}(\alpha_r(t), \beta_r(t))}{\textrm{sample from } \mathrm{Beta}(\alpha_c(t), \beta_c(t))}$} & $\alpha_r(t) = n_k(t)\avgreward + 1$ \\
         &&$\beta_r(t) = n_k(t) + 2 - \alpha_r(t)$\\
         &&$\alpha_c(t) = n_k(t)\avgcost + 1$\\
         &&$\beta_c(t) = n_k(t) + 2 - \alpha_c(t)$\\
         &&Discretization of con-\\ 
         &&tinuous values \\
         \midrule
         \multirow{2}{*}{m-UCB~\cite{xia_finite_2017}}& \multirow{2}{*}{$\frac{\min\{\avgreward + \varepsilon_k(t), 1\}}{\max\{\avgcost - \varepsilon_k(t), 0\}}$} & $\varepsilon_k(t) = \alpha\sqrt{\frac{\log(t-1)}{n_k(t)}}$\\
         &&Recommendation: $\alpha=2^{-4}$\\
         \midrule
         \multirow{2}{*}{c-UCB~\cite{xia_finite_2017}}& \multirow{2}{*}{$\frac{\avgreward}{\avgcost} + \frac{\varepsilon_k(t)}{\avgcost}$} & $\varepsilon_k(t) = \alpha\sqrt{\frac{\log(t-1)}{n_k(t)}}$\\
         &&Recommendation: $\alpha=2^{-3}$\\
         \midrule
         \multirow{2}{*}{i-UCB~\cite{xia_finite_2017}}& \multirow{2}{*}{$\frac{\avgreward}{\avgcost} + \varepsilon_k(t)$} & $\varepsilon_k(t) = \alpha\sqrt{\frac{\log(t-1)}{n_k(t)}}$ \\
         &&Recommendation: $\alpha=2^{-2}$\\
         \midrule
         \multirow{2}{*}{Budget UCB~\cite{xia_budgeted_2015}} & \multirow{2}{*}{$\frac{\avgreward}{\avgcost} + \frac{\varepsilon_k(t)}{\avgcost}\left( 
1 + \frac{\min\{ \avgreward + \varepsilon_k(t), 1 \}}{\max\{ \avgcost - \varepsilon_k(t), \lambda \}} \right)$} &$\varepsilon_k(t) = \sqrt{\frac{\log(t-1)}{n_k(t)}}$\\
&& $\lambda > 0$: minimum cost  \\
\midrule
\multirow{2}{*}{UCB-SC+~\cite{watanabe_ucb-sc_2018}}& \multirow{2}{*}[.18cm]{$\begin{array}{ll}
      \frac{\avgreward + \alpha_k(t)\avgcost}{\avgcost - \alpha_k(t)\avgreward}, & \mbox{if } \avgcost^2 > \frac{\log\frac{t}{n_k(t)}}{2n_k(t)}\\
      \infty, & \mbox{else}
    \end{array}$} & $\alpha_k(t) = \sqrt{\frac{\log \frac{t}{n_k(t)}}{2\kappa n_k(t) - \log \frac{t}{n_k(t)}}}$ \\
    && with $\kappa = \avgreward^2 + \avgcost^2$\\
         \bottomrule
    \end{tabular}
\end{table*}

\section{Additional results}

\subsection{Asymmetry of confidence interval}\label{app:asymmetry}

This section aims at provide an intuition about the level asymmetry of our confidence interval: the center of our CI is a weighted average of the sample mean and the center of the range of the random variable. This leads to CIs that are shifted towards the center of the range of the random variable. 
To see this, we can compare the distance between $\bar\mu$ and the interval center $B/2A$ to half the width of the confidence interval:
$$\mathrm{Asymmetry} = \frac{\bar\mu - \frac{B}{2A}}{\sqrt{\frac{B^2}{4A^2} - \frac{C}{A}}} \in [0,1]
$$
For Bernoulli random variables, after some derivations and inserting the definitions of $A,B,C$ as specified in Theorem~\ref{th:wilson-score}, we obtain
\begin{equation}\label{eq:asymmetry-bernoulli}
    \mathrm{Asymmetry} = \frac{(2\bar\mu-1)^2z^2}{4n\bar\mu(1-\bar\mu) + z^2}
\end{equation}

Figure 2 in the accompanying pdf plots the asymmetry measure for different values of $n$ over $\bar\mu$ and $z=3$. (1) For $\bar\mu=1$ and $\bar\mu=0$, the asymmetry takes on a maximum value of 1, while for $\bar\mu=0.5$, asymmetry is 0. (2) For a given value of $\bar\mu\in (0,1)$, asymmetry decreases with increasing sample size. (3) Related to this, we see that asymmetry is maximal for a given $\bar\mu$ for $n=1$. 

\begin{figure}[htb]
    \centering
    \includegraphics[width=.6\linewidth]{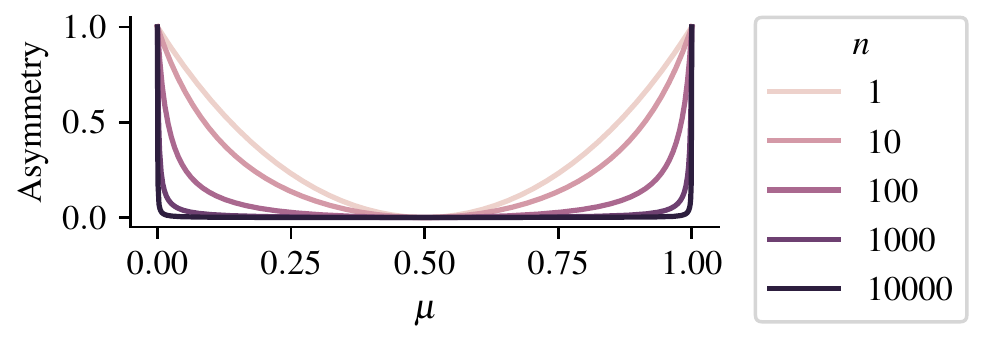}
    \caption{Asymmetry measure from \eq{eq:asymmetry-bernoulli} for Bernoulli rewards and costs for different $n$ and $\mu$ for $z=3$.}
    \label{fig:enter-label}
\end{figure}


\subsection{Evaluation results with confidence intervals}\label{sec:evaluation_results_ci}

Figure~\ref{fig:regret_evaluation_ci} shows our experimental results from Figure~\ref{fig:regret_evaluation} with 95\% confidence intervals.

\begin{figure*}[ht]
    \centering
    \begin{subfigure}[b]{\textwidth}
        \centering
        \includegraphics[width=\linewidth]{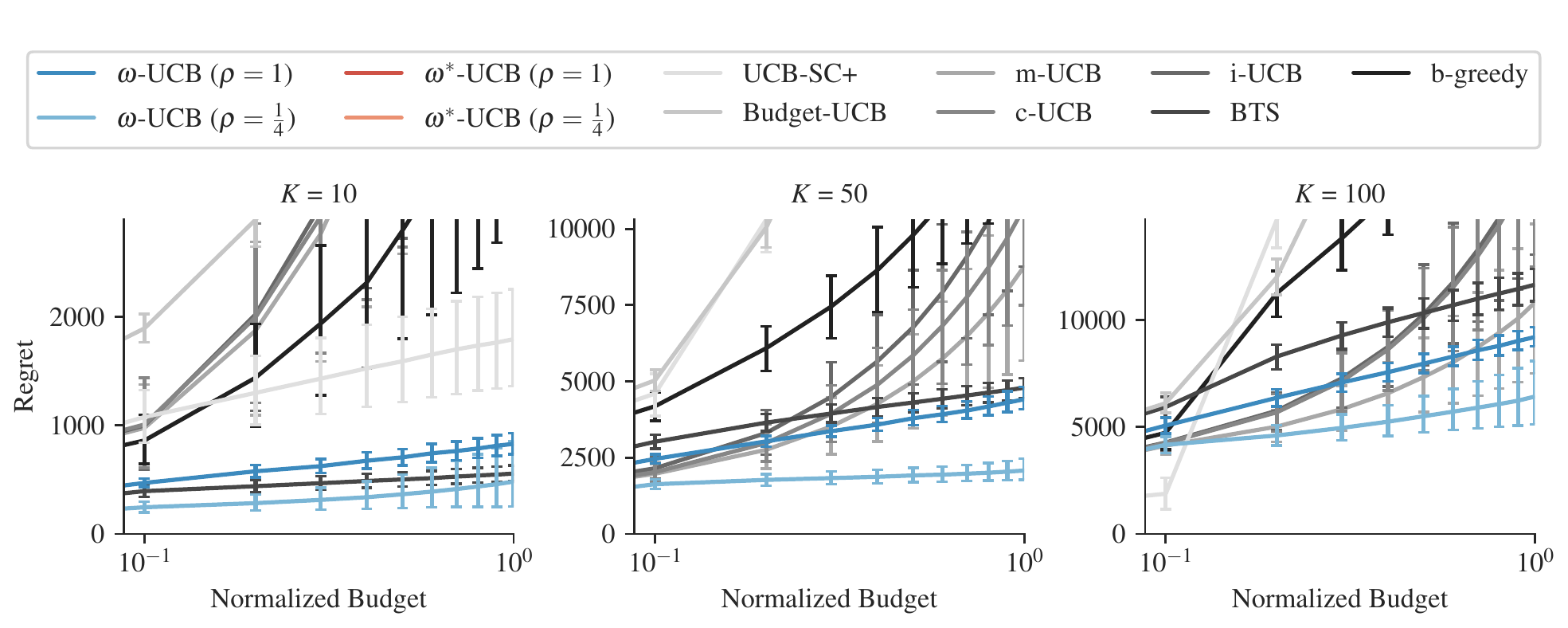}
        \caption{Settings S-Br-$\{10, 50, 100\}$}
        \label{fig:synth_bernoulli-ci}
    \end{subfigure}
    \hfill
    \begin{subfigure}[b]{\textwidth}
        \centering
        \includegraphics[width=\linewidth]{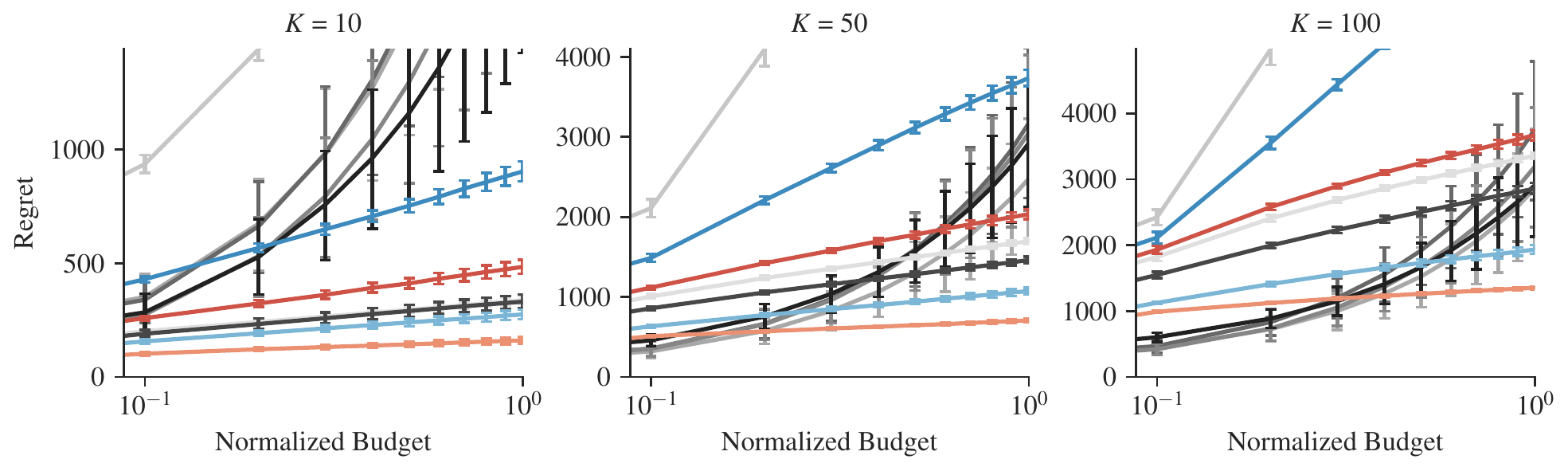}
        \caption{Settings S-GBr-$\{10, 50, 100\}$}
        \label{fig:synth_multinomial-ci}
    \end{subfigure}
    \hfill
    \begin{subfigure}[b]{\textwidth}
        \centering
        \includegraphics[width=\linewidth]{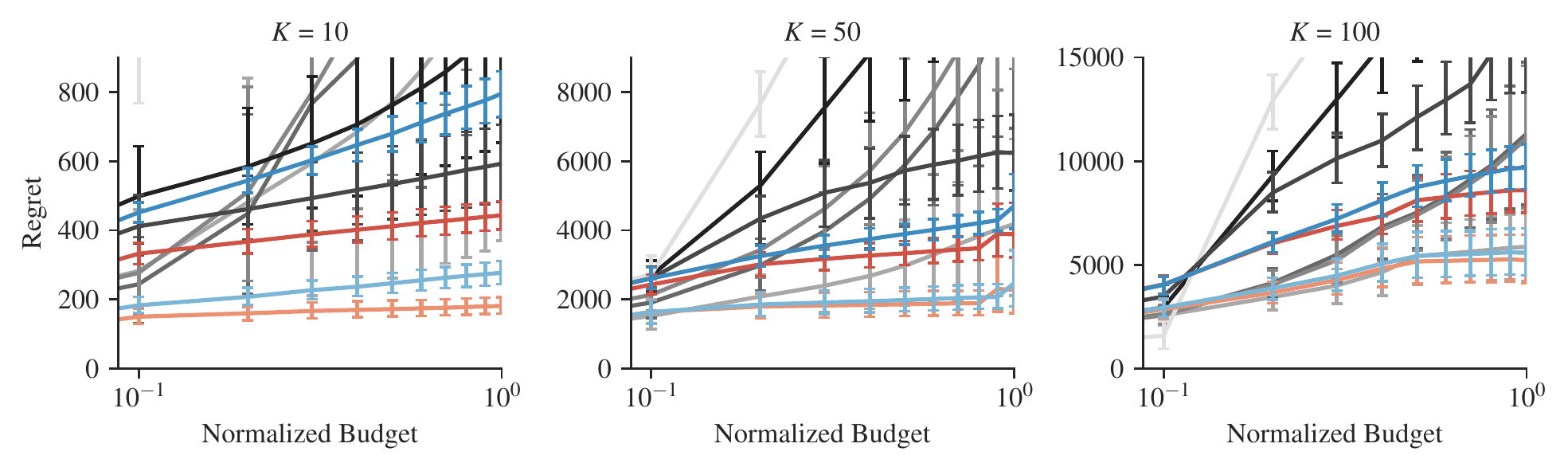}
        \caption{Settings S-Bt-$\{10, 50, 100\}$}
        \label{fig:synth_beta-ci}
    \end{subfigure}
    \centering
    \begin{subfigure}[b]{.33\textwidth}
         \centering
        \includegraphics[width=\linewidth]{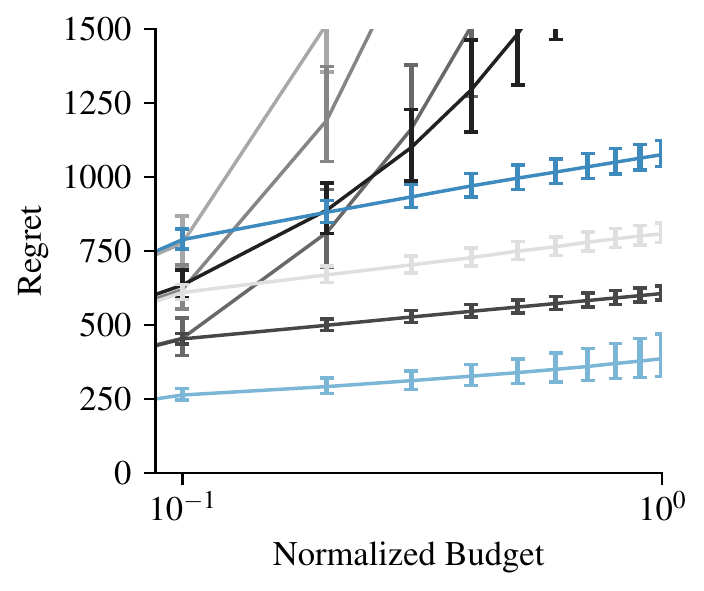}
        \caption{Setting FB-Br}
        \label{fig:facebook-bernoulli-ci}
    \end{subfigure}
    \begin{subfigure}[b]{.33\textwidth}
         \centering
          \includegraphics[width=\linewidth]{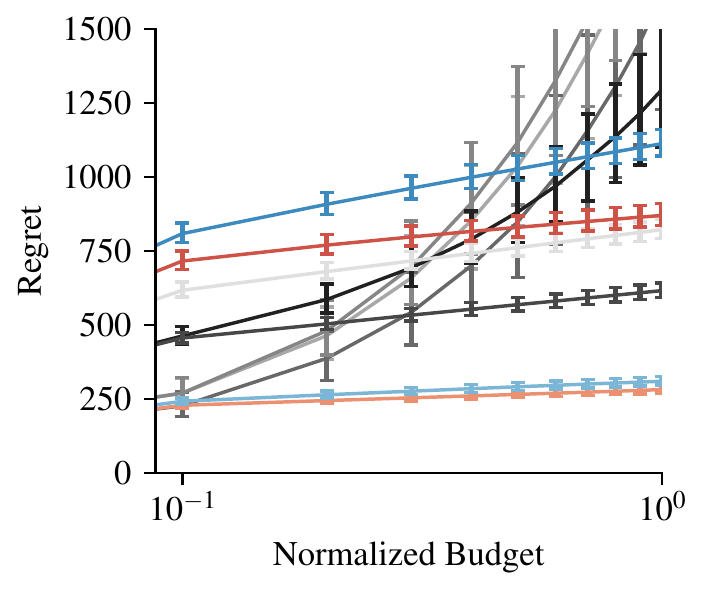}
         \caption{Setting FB-Bt}
         \label{fig:facebook-beta-ci}
    \end{subfigure}
    \begin{subfigure}[b]{.32\textwidth}
         \centering
          \includegraphics[width=\linewidth]{figures/histograms/K33_9.pdf}
         \caption{Example kde of $\expectrew,\expectcost$}
         \label{fig:_example_kde_}
    \end{subfigure}
    
    \caption{Evaluation results with 95\% confidence intervals}
\label{fig:regret_evaluation_ci}
\end{figure*}




%% file: proofs.tex
\section{Proofs and derivations}
\subsection{Proof of \texorpdfstring{Theorem~\ref{th:wilson-score}}{Theorem 1}}\label{proof:wilson-score}

Recall Theorem~\ref{th:wilson-score}:
\thwilsonscore*

\begin{proof}
Using the central limit theorem and Bhatia-Davis inequality, we follow similar steps as~\cite{wilson_probable_1927}. We first handle the case where $\mu\in(m,M)$. We then address the edge cases $\mu=m$ and $\mu=M$.

\paragraph{Case $\mu\in(m, M)$.}
The central limit theorem states that for a large enough sample size, $\bar\mu$ approximately follows a normal distribution with mean $\mu$ and variance $\frac{\sigma^2}{n}$. I.e., 
\begin{equation}
\bar\mu \sim \mathcal{N}
\left(\mu,\sqrt{\frac{\sigma^2}{n}}\right)\ \Longleftrightarrow\
    \frac{\bar\mu - \mu}{\sqrt{\frac{\sigma^2}{n}}} \sim \mathcal{N}(0,1).
\end{equation}

Therefore, $\bar\mu$ likely falls into an interval that is centered around $\mu$ and scaled by $\sigma$. The value $z$ is the number of standard deviations such that $\bar\mu$ falls out of the corresponding confidence interval with a probability of $\alpha$.
\begin{equation}\label{eq:probability_sample_mean}
    \Pr[\bar\mu \not\in \left[\mu - \frac{\sigma}{\sqrt{n}}z, \mu + \frac{\sigma}{\sqrt{n}}z\right]] = \alpha 
\end{equation}

Next, we apply the Bhatia-Davis inequality~\cite{bhatia_BetterBoundVariance_2000} to express $\sigma$ as a function of $\mu$. It states that $\sigma^2 \leq (M-\mu)(\mu-m)$. Hence there exists a factor $\eta\in [0,1]$ such that
\begin{equation}\label{eq:eta}
    \sigma^2 = \eta(M-\mu)(\mu-m).
\end{equation}

This gives us an expression for the interval bounds in \eq{eq:probability_sample_mean} that is quadratic w.r.t. $\mu$: 

\begin{equation}\label{eq:wilson-foundation}
    (\bar\mu-\mu)^2 = \frac{\sigma^2}{n}z^2 = \frac{\eta(M-\mu)(\mu-m)}{n}z^2
\end{equation}

Solving \eq{eq:wilson-foundation} for $\mu$ yields the endpoints $\omega_-(\alpha)$ and $\omega_+(\alpha)$ of our confidence interval: 
\begin{equation}
    \Pr[\mu\not\in \left[\omega_-(\alpha),\omega_+(\alpha)\right]] = \alpha
\end{equation}
with
\begin{equation}
\omega_-(\alpha), \omega_+(\alpha) = \frac{B}{2A} \pm \sqrt{\frac{B^2}{4A^2} - \frac{C}{A}}
\end{equation}
and 
\begin{equation}
A=n+z^2\eta, \quad B=2n\bar\mu + z^2\eta (M + m), \quad C=n\bar\mu^2 + z^2\eta Mm.
\end{equation}

\paragraph{Cases $\mu=m$ and $\mu=M$.} For $\mu=m$ (the case for $\mu=M$ is analogous), the probability that $\mu$ is not in the confidence interval $[\omega_-(\alpha),\omega_+(\alpha)]$ is zero, which is less than $\alpha$. Additionally, we have $\sigma^2 = (M-\mu)(\mu-m) = 0$, which implies that \eq{eq:eta} holds for any choice of $\eta$. However, since $\mu$ is an unknown quantity, we can never be certain that $\mu=m$ based on some sample from $X$. In the worst-case scenario, $X$ is a random variable that takes only extreme values, i.e., $X\in\{m, M\}$, with $\mu$ greater than and approximately equal to $m$. In this case, $\eta=1$ by definition. Hence, we define $\eta=1$ for $\mu\in{m,M}$. Combining the special case that $\mu\in\{m,M\}$ with the result from the previous paragraph gives Theorem~\ref{th:wilson-score}.

Our interval is a generalization of Wilson's score interval for Binomial proportions~\cite{wilson_probable_1927}. One can recover its original formulation by setting $\eta=1, m=0, M=1$. We refer to Section~\ref{sec:omega-ucb} 
for a more in-depth discussion of our result. 
\end{proof} 

\subsection{Proof of \texorpdfstring{Theorem~\ref{th:adaptive-conf}}{3}}\label{proof:theorem-adaptive-conf}

Recall Theorem~\ref{th:adaptive-conf}:

\thadaptiveconf*

\begin{proof}
We start from Theorem~\ref{th:wilson-score} and equate the confidence level $1 - \alpha$ of the individual reward and cost distributions to the number of standard deviations $z$. This involves the cumulative density function (cdf) of the standard normal distribution.
We then replace the cdf with an approximation that has a closed form solution for $z$. Our choice of $z$ cancels out the exponential term in this approximation, similar to the UCB1-policy~\cite{auer_finite-time_2002}. Last, we apply Theorem~\ref{cor:deviation-reward-cost} to obtain the final result.

\paragraph{Step 1.} We relate the confidence level $1 - \alpha(t)$ at time $t$ to the cumulative density function of the standard normal distribution:
\begin{equation}
    \label{eq:cdf}
    1 - \frac{\alpha(t)}{2} = \frac{1}{2}\left(1 + \text{erf}\left(\frac{z}{\sqrt{2}}\right)\right)
\end{equation}

Solving for $\alpha(t)$ yields:
\begin{equation}
    \alpha(t) = 1 - \mathrm{erf}\left(\frac{z}{\sqrt{2}}\right)
\end{equation}

\paragraph{Step 2.}
We now replace the error function $\text{erf}\left(\frac{z}{\sqrt{2}}\right)$ in the equation above with a series expansion based on B\"urmann's theorem~\cite{schopf_burmanns_2014,whittaker_CourseModernAnalysis_2020}; we summarize all but the first addend in a remainder term $\gamma\left(\frac{z}{\sqrt{2}}\right) > 0$. This term has a maximum of $\gamma(0.71) \approx 0.0554$ and approaches 0 for larger $z$:
\begin{equation}
    \alpha(t) = 1 - \left(\sqrt{1-\exp{-\frac{z^2}{2}}} + \gamma\left(\frac{z}{\sqrt{2}}\right)\right)
\end{equation}

Omitting the $\gamma$-term gives an upper bound for $\alpha(t)$:
\begin{equation}\label{eq:alpha-exp}
    \alpha(t) < 1 - \sqrt{1-\exp{-\frac{z^2}{2}}}
\end{equation}

\paragraph{Step 3.}
Next, we choose $z$ as a function of $\log t$ and $\rho>0$, $z_\rho(t)=\sqrt{2\rho\log t}$. This results in a time-increasing confidence level $1 - \alpha(t)$:
\begin{equation}\label{eq:alpha_t}
    \Pr[\mu \not\in \left[\omega_-(\alpha(t)),\omega_+(\alpha(t))\right]] \leq \alpha(t) \text{ with } \alpha(t) < 1 - \sqrt{1-t^{-\rho}}
\end{equation}

\paragraph{Step 4.}
Applying Theorem~\ref{cor:deviation-reward-cost} gives
\begin{equation}
    \Pr[\frac{\expectrew}{\expectcost} > \Omega_k(\alpha, t)] \leq \alpha(t) \text{ with } \alpha(t) < 1 - \sqrt{1-t^{-\rho}}.
\end{equation}

The complementary event, $\Omega_k(\alpha, t) \geq \expectrew/\expectcost$, holds with a probability of at least $1 - \alpha(t)$,
\begin{equation}
    \Pr[\Omega_k(\alpha, t) \geq \frac{\expectrew}{\expectcost}] \geq 1 - \alpha(t) \text{ with } \alpha(t) < 1 - \sqrt{1-t^{-\rho}},
\end{equation}
which is the result given in Theorem~\ref{th:adaptive-conf}.

\end{proof}

\subsection{Proof of \texorpdfstring{Theorem~\ref{th:suboptimal-plays}}{Theorem 4}}\label{proof:th_suboptimal-plays}

Recall Theorem~\ref{th:suboptimal-plays}:

\thsuboptimalplays*

\begin{proof}
The proof starts with a general expression for the number of plays of a suboptimal arm $k>1$, where $\mathds{1}\{\cdot\}$ denotes the indicator function.

\begin{equation}\label{eq:decomposition-start}
    n_k(\tau) \leq 1 + \sumtau \mathds{1}\left\{ \Omega_k(t) \geq \Omega_j(t), \forall j\neq i \right\} \leq 1 + \sumtau \mathds{1}\left\{ \Omega_k(t) \geq \Omega_1(t) \right\}
\end{equation}

This is upper-bounded by
\begin{alignat}{1}\label{eq:decomposition-2}
    n_k(\tau) &\leq 1 + \sumtau \bigg[\mathds{1}\left\{ \Omega_k(t) \geq \Omega_1(t), \Omega_1(t) < \frac{\mu_1^r}{\mu_1^c} \right\} + \mathds{1}\left\{ \Omega_k(t) \geq \Omega_1(t), \Omega_1(t) \geq \frac{\mu_1^r}{\mu_1^c}\right \}\bigg] \\
    &\leq 1 + \sumtau \bigg[ \mathds{1}\left\{ \Omega_1(t) < \frac{\mu_1^r}{\mu_1^c} \right\} + \mathds{1}\left\{ \Omega_k(t) \geq \frac{\mu_1^r}{\mu_1^c}\right\} \bigg]
\end{alignat}

The expected number of plays $\mathbb{E}[n_k(\tau)]$ is given by the probabilities of the individual events: 

\begin{equation}\label{eq:number_of_plays_with_A_B}
    \mathbb{E}[n_k(\tau)] \leq 1 + \sumtau \bigg[\underbrace{\Pr\left\{ \Omega_1(t) < \frac{\mu_1^r}{\mu_1^c} \right\}}_{\Pr[A]} + \underbrace{\Pr\left\{ \Omega_k(t) \geq \frac{\mu_1^r}{\mu_1^c}\right\}}_{\Pr[B]} \bigg].
\end{equation}

We now evaluate the sum in the equation above.

\paragraph{Sum of $\Pr[A]$.}
We apply Theorem~\ref{cor:deviation-reward-cost}:
\begin{equation}
\label{eq:A}
    \sumtau \Pr[A] < \sumtau \left[1 - \sqrt{1-t^{-\rho}}\right] = (\tau-K)-\sumtau \sqrt{1-t^{-\rho}}
\end{equation}

\paragraph{Sum of $\Pr[B]$.} For this step, let us first introduce a helpful lemma: Lemma~\ref{lemma:t_k_star} bounds $\Pr[\Omega_k(t)\geq \mu_1^r/\mu_1^c]$ after a minimum number of plays $n_k^*(\tau)$, which grows logarithmic with $\tau$.

\begin{restatable}[]{lemma}{lemmatistar}
\label{lemma:t_k_star}
Define $\delta_k=\frac{\Delta_k}{\Delta_k + 1/\expectcost}$ and $n_k^*(\tau)$ as follows:

\begin{equation}
n_k^*(\tau) = \frac{8\rho\log \tau}{\delta_k^2}\max\left\{ \frac{\eta_k^r\expectrew}{1-\expectrew}, \frac{\eta_k^c(1-\expectcost)}{\expectcost} \right\}
\end{equation}

The following inequality holds whenever $n_k(t)\geq n_k^*(\tau)$:

\begin{equation}\label{eq:lemma_t_k_star}
   \Pr[\Omega_k(t) \geq \frac{\mu_1^r}{\mu_1^c}] < 1 - \sqrt{1-\tau^{-\rho}}
\end{equation}
\end{restatable}

Appendix~\ref{proof:lemma_t_k_star} contains the proof of Lemma~\ref{lemma:t_k_star}. See Appendix~\ref{proof:delta-gap} for the derivation of $\delta_k$.

The lemma allows to decompose $\sumtau\Pr[B]$ into ``initial plays'' ($n_k(t) < n_k^*(\tau)$) and ``later plays'' ($n_k(t) \geq n_k^*(\tau)$):
\begin{equation}
    \sumtau\Pr[B] = \sumtau\Pr\left\{ \Omega_k(t) \geq \frac{\mu_1^r}{\mu_1^c}\right\} = n_k^*(\tau) + \sumtau \Pr\left\{ \Omega_k(t) \geq \frac{\mu_1^r}{\mu_1^c}, n_k(\tau) \geq n_k^*(\tau)\right\}
\end{equation}

We now apply Lemma~\ref{lemma:t_k_star} to evaluate the sum in above equation,
\begin{equation}
\sumtau\Pr[B] \leq n_k^*(\tau) + \sumtau\left[ 1-\sqrt{1-\tau^{-\rho}} \right] = n_k^*(\tau) + (\tau - K)\left(1 - \sqrt{1-\tau^{-\rho}}\right). \label{eq:B}
\end{equation}

Inserting the results of \eq{eq:A} and \eq{eq:B} in \eq{eq:number_of_plays_with_A_B} yields a bound on $\mathbb{E}[n_k(\tau)]$:
\begin{gather}
    \mathbb{E}[n_k(\tau)] \leq 1 + n_k^*(\tau) + \underbrace{\left(\tau-K\right)\left(2-\sqrt{1-\tau^{-\rho}}\right) - \sumtau\sqrt{1-t^{-\rho}}}_{\xi(\tau, \rho)}
\end{gather}
\end{proof}

\subsection{Derivation of Proportional \texorpdfstring{$\delta$}{delta}-gap}\label{proof:delta-gap}

Let $\delta_k$ be the proportional $\delta$-gap of arm $k$. It measures how much one must increase $\expectrew$ and decrease $\expectcost$ in order to make arm $k$ have the same reward-cost ratio as arm 1, or in other words, to bridge the suboptimality gap $\Delta_k$ between arm $k$ and arm 1. The $\delta$-gap is proportional to the possible increase in rewards (decrease in costs) without violating their range $[0,1]$. We start our derivation with \eq{eq:delta-gap-der-1}, which states this mathematically.
\begin{equation}\label{eq:delta-gap-der-1}
    \frac{\expectrew + \delta_k(1-\expectrew)}{\expectcost - \delta_k\expectcost} = \frac{\mu_1^r}{\mu_1^c}
\end{equation}

Rearranging the equation yields 
\begin{equation}
    \frac{\mu_1^r}{\mu_1^c} - \frac{\expectrew}{\expectcost} = \delta_k\left( \frac{\mu_1^r}{\mu_1^c} - \frac{\expectrew}{\expectcost} + \frac{1}{\expectcost} \right).
\end{equation}

Now, recall the definition of an arm's suboptimality, $\Delta_k = \frac{\mu_1^r}{\mu_1^c} - \frac{\expectrew}{\expectcost}$, and solve for $\delta_k$:
\begin{equation}
     \delta_k = \frac{\Delta_k}{\Delta_k + \frac{1}{\expectcost}}
\end{equation}

\subsection{Proof of \texorpdfstring{Lemma~\ref{lemma:t_k_star}}{Lemma 1}}\label{proof:lemma_t_k_star}

Recall Lemma~\ref{lemma:t_k_star}:
\lemmatistar*

\begin{proof}
The proof is structured in four steps. First, we derive an expression for the maximum deviation between the observed sample mean and the unknown expected value of an arm's rewards and costs that we use later on. Second, we decompose the probability $\Pr[\Omega_k(t)\geq \mu_1^r/\mu_1^c]$. Third, we evaluate the decomposed probabilities for cases where $n_k(t)$ is sufficiently large, that is, $n_k(t) \geq n_k^*(\tau)$. Finally, we recombine the decomposed probabilities to obtain the final result.

\paragraph{Deviation between sample mean and expected value.} Let $\bar \mu_k(t)$ be the sample mean and $\mu_k$ the expected value of arm $k$'s rewards or costs at time $t$. To quantify the deviation between mean $\bar \mu_k(t)$ and $\mu_k$ we start from \eq{eq:wilson-foundation} (central limit theorem) and set $[m,M]=[0,1]$ and $n=n_k(t)$):
\begin{equation}
\label{eq:epsilon}
    (\bar \mu_k(t)-\mu_k)^2 \leq \frac{\eta_k\mu_k(1-\mu_k)}{n_k(t)}z^2
\end{equation}

The above inequality holds with the same probability as our confidence interval since it is the basis of the confidence interval derivation.\footnote{This observation is commonly known as ``interval equality principle''.}
This allows us to bound the deviation between sample mean and expected value, denoted $\varepsilon_{k}(t)$, for our choice of $z_\rho(t)=\sqrt{2\rho\log t}$:
\begin{equation}
\Pr[\lvert \bar\mu_k(t) - \mu_k\rvert > \varepsilon_{k}(t)] \leq \alpha(t) \text{ with } \varepsilon_{k}(t) = \sqrt{\frac{2\eta_k\mu_k(1-\mu_k)\rho\log t}{n_k(t)}} \text{ and } \alpha(t) < 1 - \sqrt{1-t^{-\rho}}
    \label{eq:epsilon3}
\end{equation}

Whenever we refer to $\varepsilon_{k}(t)$ w.r.t. rewards or costs we use the notations $\varepsilon_{k}^r(t)$ and $\varepsilon_{k}^c(t)$.

\paragraph{Decomposition of $\Pr[\Omega_k(t) \geq {\mu_1^r}/{\mu_1^c}]$.}

Next, we decompose the probability that $\Omega_k(t)\geq {\mu_1^r}/{\mu_1^c}$. The decomposition is analogous to the one in the proof of Theorem~\ref{cor:deviation-reward-cost} and thus omitted here: 
\begin{align}
    \Pr[\Omega_k(t) \geq \frac{\mu_1^r}{\mu_1^c}] &= \Pr[\rcratio \geq \frac{\mu_1^r}{\mu_1^c}] = \Pr[\rcratio \geq \frac{\expectrew + (1-\expectrew)\delta_k}{\expectcost - \expectcost\delta_k}] \label{eq:decomposition_Omega_geq_rcbest} \\ 
    &\leq \Pr[\omegaplus \geq \expectrew + (1-\expectrew)\delta_k] + \Pr[\omegaminus \leq \expectcost - \expectcost\delta_k] \label{eq:decomposition_Omega_geq_rcbest_final}
\end{align}

For the next step, note that if $\expectrew \leq \omega_{k+}^r(t)$ (and $\expectcost \geq \omega_{k-}^c(t)$), we have that $\expectrew - \avgreward \leq \varepsilon^r_{k}(t)$ (and $\avgcost - \expectcost \leq \varepsilon_{k}^c(t)$). This allows us to rewrite the terms in \eq{eq:decomposition_Omega_geq_rcbest_final} as follows:
\begin{equation}\label{eq:omega_plus_over_mu_1_r}
\begin{split}
    \Pr[\omegaplus \geq \expectrew + (1-\expectrew)\delta_k] &= \Pr[\expectrew + (1-\expectrew)\delta_k - \avgreward \leq \varepsilon_{k}^r(t)]\\
    &= \Pr[\avgreward-\expectrew \geq (1-\expectrew)\delta_k - \varepsilon_{k}^r(t)]
\end{split}
\end{equation}
\begin{equation}\label{eq:omega_plus_over_mu_1_c}
\begin{split}
    \Pr[\omegaminus \leq \expectcost - \expectcost\delta_k] &= \Pr[\avgcost - (\expectcost - \expectcost\delta_k) \leq \varepsilon_{k}^c(t)]\\
    &=\Pr[\expectcost - \avgcost \geq \expectcost\delta_k - \varepsilon_{k}^c(t)]
\end{split}
\end{equation}

In the next two paragraphs, we evaluate \eq{eq:omega_plus_over_mu_1_r} (confidence bound of rewards) and \eq{eq:omega_plus_over_mu_1_c} (confidence  bound of costs). We combine both results afterwards.
\paragraph{Evaluation of \eq{eq:omega_plus_over_mu_1_r} for $n_k(t) \geq n_k^{*,r}(\tau)$.}
We now consider the cases in which the number of times arm $k$ was played is at least logarithmic w.r.t. $
\tau$, i.e., $n_k(t) \geq n_k^{*,r}(\tau)$ with
\begin{equation}
    n_k^{*,r}(\tau) = \frac{2\rho\log \tau}{\delta_k^2(1-\kappa)^2} \frac{\eta_k^r\expectrew}{1-\expectrew}, \quad \text{for any } \kappa \in (0,1).
\end{equation}

In those cases, $\varepsilon_{k}^r(t) \leq (1-\kappa)\delta_k(1-\expectrew)$. One can verify this by inserting $n_k^{*,r}(\tau)$ in the definition of $\varepsilon_k^r(t)$, cf.~\eq{eq:epsilon3}. This gives the following inequality for the right side of \eq{eq:omega_plus_over_mu_1_r}:
\begin{equation}\label{eq:proof_lemma_prob_epsilon_trick}
    \Pr[\omegaplus \geq \expectrew + (1-\expectrew)\delta_k] \leq \Pr[\avgreward - \expectrew \geq \kappa(1-\expectrew)\delta_k]
\end{equation}

Last, we compute the number of standard deviations $z^*$ that corresponds to a deviation between $\avgreward$ and $\expectrew$ of at maximum $\kappa(1-\expectrew)\delta_k$ based on \eq{eq:epsilon}. In particular, we solve the right-most inequality in the expression below: 
\begin{equation}
(\avgreward-\expectrew)^2 \leq \frac{\eta_k^r\expectrew(1-\expectrew)}{n_k(\tau)}{z^*}^2 \leq
\frac{\eta_k^r\expectrew(1-\expectrew)}{n_k^{*,r}(\tau)}{z^*}^2 \leq (\kappa(1-\expectrew)\delta_k)^2
\end{equation}

With our choice of $n_k^{*,r}(\tau)$, this yields $z^* = \left(2\rho\log \tau \kappa^2 (1-\kappa)^{-2}\right)^\frac{1}{2}$. Inserting $z^*$ in \eq{eq:alpha-exp} (upper bound for $\alpha(t)$) results in the following bound: 
\begin{equation}\label{eq:omega_plus_over_mu_1_r_final}
     \Pr[\omegaplus \geq \expectrew + (1-\expectrew)\delta_k] < \frac{1}{2}\left(1 - \sqrt{1 - \tau^{-\frac{\kappa^2\rho}{(1-\kappa)^2}}}\right)
\end{equation}

\paragraph{Evaluation of \eq{eq:omega_plus_over_mu_1_c} for $n_k(t) > n_k^{*,c}(\tau)$.}
Again, we consider the cases in which the number of times arm $k$ was played is at least logarithmic in $
\tau$, i.e., $n_k(t) \geq n_k^{*,c}(\tau)$ with
\begin{equation}
    n_k^{*,c}(\tau) = \frac{2\rho\log \tau}{\delta_k^2(1-\kappa)^2} \frac{\eta_k^c(1-\expectcost)}{\expectcost}, \quad \kappa \in (0,1).
\end{equation}

In those cases, $\varepsilon_{k}^c(t) \leq (1-\kappa)\delta_k\expectcost$. Following analogous steps as in the previous paragraph yields \eq{eq:omega_plus_over_mu_1_c_final}:
\begin{equation}\label{eq:omega_plus_over_mu_1_c_final}
     \Pr[\omegaminus \leq \expectcost - \expectcost\delta_k] < \frac{1}{2}\left(1 - \sqrt{1 - \tau^{-\frac{(1-\kappa)^2\rho}{\kappa^2}}}\right)
\end{equation}

\paragraph{Obtaining the final result.} With the results in \eq{eq:omega_plus_over_mu_1_r_final} and \eq{eq:omega_plus_over_mu_1_c_final} we can finally evaluate \eq{eq:decomposition_Omega_geq_rcbest}. A choice of $\kappa=0.5$ and 
\begin{equation}
n_k^*(\tau) = \frac{8\rho\log \tau}{\delta_k^2}\max\left\{ \frac{\eta_k^r\expectrew}{1-\expectrew}, \frac{\eta_k^c(1-\expectcost)}{\expectcost} \right\}
\end{equation}
yields the bound given in Lemma~\ref{lemma:t_k_star}:
\begin{equation}
    \Pr[\Omega_k(t) \geq \frac{\mu_1^r}{\mu_1^c}] < 1 - \sqrt{1 - \tau^{-\rho}}, \quad n_k(t) \geq n_k^*(\tau)
\end{equation}
\end{proof}

\subsection{Proof of \texorpdfstring{Theorem~\ref{th:worst-case-regret-complex}}{Theorem 6}}\label{proof:worst-case-regret-complex}

Recall Theorem~\ref{th:worst-case-regret-complex}:
\thworstcaseregretcomplex*

\begin{proof}
First note that the latter two terms and $n_k^*(\tau_B)$ in \eq{eq:regret-wucb} are in $\mathcal{O}(\log B)$~\cite{xia_finite_2017}.
Next we show that $\xi(\tau_B,\rho)$ is in $\mathcal{O}(\log B)$ for $\rho \geq 1$ and in $\mathcal{O}(B^{1-\rho})$ for $0 < \rho < 1$. 

We exploit two inequalities in our proof; the latter is based on the integral test for convergence and holds for continuous, positive, decreasing functions. 
\begin{equation} \label{eq:inequality-roots}  
    \sqrt{1-t^{-\rho}} \geq 1 - t^{-\rho}, \quad t\geq 1,\rho>0
\end{equation}
\begin{equation}\label{eq:inequality-integral}
    \sum_{t=K+1}^{\tau_B} t^{-\rho} \leq (K+1)^{-\rho} + \int_{t=K+1}^{\tau_B} t^{-\rho}\ dt
\end{equation}

We use \eq{eq:inequality-roots} to obtain an integrable expression for the sum in $\xi(\tau_B,\rho)$. We replace the sum with an integral-based upper bound as in \eq{eq:inequality-integral}:
\begin{align}
    \xi(\tau_B,\rho) &= \left(\tau_B-K\right)\left(2-\sqrt{1-\tau_B^{-\rho}}\right) - \sum_{t=K+1}^{\tau_B}\sqrt{1-t^{-\rho}}\label{eq:xi-bound-1}\\
    &\leq \left(\tau_B-K\right)\left(2-\sqrt{1-\tau_B^{-\rho}}\right) - \sum_{t=K+1}^{\tau_B} 1 - t^{-\rho}\\
    &= \left(\tau_B-K\right)\left(1-\sqrt{1-\tau_B^{-\rho}}\right) + \sum_{t=K+1}^{\tau_B} t^{-\rho}\\
    &\leq \left(\tau_B-K\right)\left(1-\sqrt{1-\tau_B^{-\rho}}\right) + (K+1)^{-\rho} + \int_{t=K+1}^{\tau_B}t^{-\rho}\ dt \label{eq:xi-pre-integral-eval}
\end{align}

Next, we evaluate the integral for the cases $\rho = 1$ and $\rho \neq 1$.

\paragraph{Case 1: $\rho=1$.}
\begin{align}
    \xi(\tau_B,\rho) \leq \left(\tau_B-K\right)\left(1-\sqrt{1-\tau_B^{-1}}\right) + (K+1)^{-1} + \log \tau_B - \log (K+1)
\end{align}

For $\rho=1$, the first term in above equation converges, so $\xi(\tau_B,\rho=1)$ is in $\mathcal{O}(\log B)$. This implies that the overall regret of $\omega$-UCB is in $\mathcal{O}(\log B)$ for $\rho=1$.

\paragraph{Case 2: $\rho \neq 1$.}
\begin{align}
    \xi(\tau_B,\rho) \leq \left(\tau_B-K\right)\left(1-\sqrt{1-\tau_B^{-\rho}}\right) + (K+1)^{-\rho} + \frac{1}{1-\rho}\left( \tau_B^{1-\rho} - (K+1)^{1-\rho} \right)
\end{align}

For $\rho > 1$, $\xi(\tau_B,\rho)$ converges and thus the overall regret is in $\mathcal{O}(\log B)$. For $0<\rho<1$, one can show that $\xi(\tau_B,\rho)$ is in $\mathcal{O}(B^{1-\rho})$. Hence, the regret of $\omega$-UCB is in $\mathcal{O}(B^{1-\rho})$ for $0<\rho<1$.

To summarize, the regret of our policy is in $\mathcal{O}(B^{1-\rho})$ for $0<\rho<1$ and in $\mathcal{O}(\log B)$ for $\rho\geq 1$. 
\end{proof}